%% file: main.tex
\pdfoutput=1
\documentclass[iicol,pdflatex,sn-chicago]{sn-jnl}

\usepackage{bm}
\usepackage{amsmath,amssymb,amsfonts,amstext,amsthm}
\usepackage{mathtools}
\usepackage{multirow}
\usepackage{booktabs}
\usepackage{graphicx}
\usepackage[ruled,linesnumbered,noend]{algorithm2e}
\usepackage{array}
\usepackage{url}
\usepackage{hyperref}

\newcommand{\journaldoi}{10.1007/s10514-026-10269-0}

\graphicspath{{figures/}}

\usepackage{caption}
\usepackage{subcaption}
\usepackage[table,xcdraw]{xcolor}

\usepackage{nicematrix}
\usepackage{siunitx}
\usepackage[normalem]{ulem}
\robustify\bfseries
\robustify\uline
\def\Decimal{.000}

\def\Ulinehelp#1.#2 {
  #1.#2\setbox0=\hbox{#1\Decimal}\hspace{-\wd0}{\if\relax#2\relax%
    \uline{\phantom{#1.0}}\else\uline{\phantom{#1.#2}}\fi}%
    }

\usepackage{threeparttable}
\usepackage{etoolbox}
\appto\TPTnoteSettings{\footnotesize}

\theoremstyle{thmstyleone}
\newtheorem{theorem}{Theorem}
\newtheorem{proposition}{Proposition} 

\theoremstyle{thmstyleone}

\newtheorem{remark}{Remark}

\theoremstyle{thmstylethree}
\newtheorem{definition}{Definition}

\definecolor{red(ncs)}{rgb}{0.77, 0.01, 0.2}

\hypersetup{
  pdftitle={Greedy Heuristics for Sampling-Based Motion Planning in High-Dimensional State Spaces},
  pdfauthor={Phone Thiha Kyaw, Anh Vu Le, Rajesh Elara Mohan, Jonathan Kelly},
  pdfsubject={Author accepted manuscript. Autonomous Robots 50, 44 (2026). https://doi.org/\journaldoi},
  pdfkeywords={Sampling-based motion planning, optimal path planning, informed sampling, bidirectional search, greedy heuristics, high-dimensional planning}
}

\DeclareUnicodeCharacter{2217}{*}

\renewcommand{\orcidlogo}{
  \includegraphics[width=10pt]{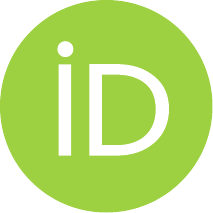}
}
\renewcommand{\orcid}[1]{\href{#1}{\orcidlogo}}

\makeatletter
\newcommand\thefontsize{The current font size is: \f@size pt}
\makeatother

\makeatletter
\let\orig@maketitle\@maketitle
\renewcommand{\@maketitle}{\vspace{-41pt}\orig@maketitle}
\makeatother

\makeatletter
\patchcmd{\@algocf@start}
  {-1.5em}
  {0pt}
  {}{}
\makeatother

\SetArgSty{textup}

\usepackage{xparse}

\let\originalleft\left
\let\originalright\right
\renewcommand{\left}{\mathopen{}\mathclose\bgroup\originalleft}
\renewcommand{\right}{\aftergroup\egroup\originalright}

\NewDocumentCommand\bbm{}{ \begin{bmatrix} }
\NewDocumentCommand\ebm{}{ \end{bmatrix} }

\NewDocumentCommand\Vector{m}{ \boldsymbol{\mathbf{#1}} }

\NewDocumentCommand\Norm{m}{ \left\Vert#1\right\Vert }

\NewDocumentCommand{\set}{g}{%
  \mathcal{X}\IfNoValueTF{#1}{}{_{#1}}%
}

\newcommand{\infset}{\set{\Hat{f}}}
\newcommand{\ginfset}{\set{\mathrm{greedy}}}
\newcommand{\xrand}{\mathbf{x}_\text{rand}}

\begin{document}

\title{Greedy Heuristics for Sampling-Based Motion Planning in High-Dimensional State Spaces}

\author*[1]{\fnm{Phone Thiha} \sur{Kyaw} \orcid{https://orcid.org/0000-0001-8790-8350}} \email{phone.thiha@robotics.utias.utoronto.ca}

\author[2]{\fnm{Anh Vu} \sur{Le} \orcid{https://orcid.org/0000-0002-4804-7540}} \email{leanhvu@tdtu.edu.vn}

\author[3]{\fnm{Rajesh Elara} \sur{Mohan} \orcid{https://orcid.org/0000-0001-6504-1530}} \email{rajeshelara@sutd.edu.sg}

\author[1]{\fnm{Jonathan} \sur{Kelly} \orcid{https://orcid.org/0000-0002-5528-6136}} \email{jonathan.kelly@robotics.utias.utoronto.ca}

\affil[1]{\orgdiv{Space \& Terrestrial Autonomous Robotic Systems (STARS) Laboratory}, \orgname{University of Toronto Institute for Aerospace Studies}, \orgaddress{\street{4925 Dufferin Street}, \city{Toronto}, \postcode{M3H 5T6}, \state{Ontario}, \country{Canada}}}

\affil[2]{\orgdiv{Advanced Intelligent Technology Research Group}, \orgname{Faculty of Electrical and Electronics Engineering, \mbox{Ton Duc Thang University}}, \orgaddress{\city{Ho Chi Minh City}, \postcode{700000}, \country{Vietnam}}}

\affil[3]{\orgdiv{ROAR Lab}, \orgname{Engineering Product Development, Singapore University of Technology and Design}, \orgaddress{\city{Singapore}, \postcode{487372}, \country{Singapore}}}

\abstract{
Informed sampling techniques accelerate the convergence of sampling-based motion planners by biasing sampling toward regions of the state space that are most likely to yield better solutions.
However, when the current solution path contains redundant or tortuous segments, the resulting informed subset may remain unnecessarily large, slowing convergence.
Our prior work addressed this issue by introducing the greedy informed set, which reduces the sampling region based on the maximum heuristic cost along the current solution path.
In this article, we formally characterize the behavior of the greedy informed set within Rapidly-exploring Random Tree (RRT*)-like planners and analyze how greedy sampling affects exploration and asymptotic optimality.
We then present Greedy RRT* (G-RRT*), a bi-directional anytime variant of RRT* that leverages the greedy informed set to focus sampling in the most promising regions of the search space.
Experiments on abstract planning benchmarks, manipulation tasks from the MotionBenchMaker dataset, and a dual-arm Barrett WAM problem demonstrate that G-RRT* rapidly finds initial solutions and converges asymptotically to optimal paths, outperforming state-of-the-art sampling-based planners.
}

\keywords{Sampling-based motion planning, optimal path planning, informed sampling, bidirectional search, greedy
heuristics, high-dimensional planning}
\maketitle

\input{section1/introduction}
\input{section2/related_work}
\input{section3/preliminaries}
\input{section4/greedy_informed_set}
\input{section5/grrt_star}
\input{section6/analysis}
\input{section7/experiments}
\input{section8/conclusion}

\backmatter

\bmhead{Acknowledgements}
The authors would like to thank Yunfan Lu for dedicating time to assist with the proofs in this manuscript.

\bibliography{main}

\end{document}

%% file: section1/introduction.tex
\section{Introduction}

Path planning is the problem of finding a collision-free path from an initial state to a goal state while also considering specific optimization objectives, such as minimizing path length or energy use, for example~\citep{lavalle2006planning}.
Many planning algorithms exist, including graph search, artificial potential fields, and sampling-based methods~\citep{elbanhawi2014sampling,orthey2024sampling,zhang2025motion}.
However, it remains challenging to find collision-free, optimal paths, especially in high-dimensional state spaces; the general path planning problem is known to be PSPACE-hard~\citep{reif1979complexity}.

Sampling-based planners, such as the probabilistic roadmap~\citep[PRM]{kavraki1996probabilistic} and rapidly-exploring random tree~\citep[RRT]{lavalle2001randomized} algorithms, tackle the complexity of high-dimensional path planning by sacrificing completeness for efficiency, providing only probabilistic guarantees.
The asymptotically optimal variants PRM* and RRT*~\citep{karaman2011sampling} improve solution quality over time but remain computationally expensive due to their reliance on random sampling~\citep{gammell2021asymptotically}.
A substantial portion of the computational effort is often wasted exploring irrelevant portions of the state space.
To improve planning performance, it is crucial to focus sampling on promising regions of the problem domain (Figure~\ref{fig:greedy_informed_set_steps}).

Existing direct informed sampling methods \citep{gammell2014informed,gammell2018informed} mitigate inefficiency by defining bounded, hyperellipsoidal sampling regions, called \emph{informed sets}, based on the cost of the current solution.
While this significantly reduces the size of the search space, initial feasible solutions are often highly tortuous, containing many redundant states---intermediate states that can be removed through standard path-shortening techniques without affecting feasibility.
Such redundant states increase the solution cost and enlarge the informed set unnecessarily, reducing the likelihood of finding states that may be part of a better solution.
To address this shortcoming, \cite{kyaw2022energy} introduced a new direct informed sampling procedure that biases sampling based on heuristic information from the states that are part of the current solution, independent of their cost.
These states are used to define an alternative \emph{greedy informed set} that reduces the size of the informed sampling hyperellipsoid, improving both the efficiency and the convergence rate of planning.
Although the greedy informed set concept was introduced in \citep{kyaw2022energy}, no analysis of its properties was carried out.

\begin{figure}
    \centering
    \includegraphics[width=0.99\linewidth]{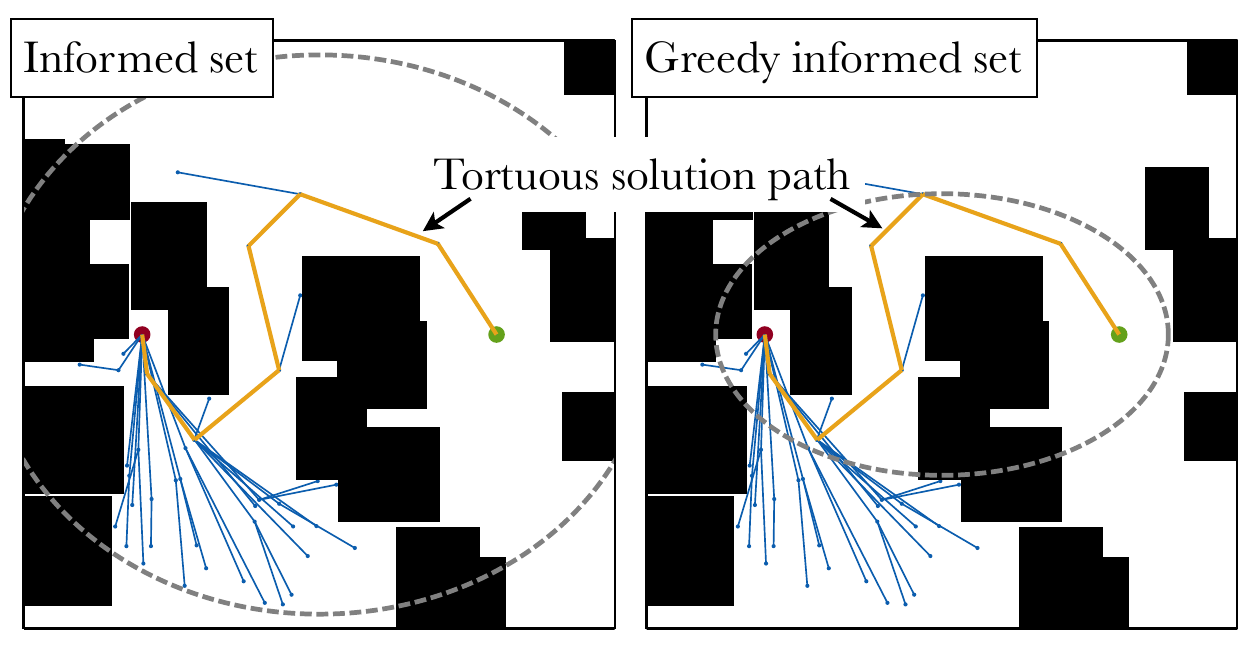}
\caption{Comparison of informed set sizes for a planning problem with a tortuous solution path (yellow) from the start (red) to the goal (green). The $L^2$ informed set (left) is defined by the current solution cost and covers a large ellipsoidal region. In contrast, the $L^2$ \emph{greedy informed set} (right) is defined using the state with the maximum admissible heuristic along the current solution path, substantially reducing the ellipsoidal area.}
    \label{fig:greedy_informed_set_steps}
\end{figure}

Informed sets are only useful to reduce the size of the search space once an initial solution has been found, making rapid discovery of a feasible path crucial for overall planning efficiency.
A key contribution of this work is a formal analysis of the performance and trade-offs associated with the \emph{greedy informed set}.
To accelerate the search for an initial solution, bidirectional planning can be employed.
Algorithms such as RRT-Connect~\citep{kuffner2000rrt} grow two trees---one from the start and one from the goal---that expand toward each other, guided by a connection heuristic. 
Building on this idea and incorporating the greedy informed set, we propose \emph{Greedy-RRT*} (G-RRT*), a bidirectional, asymptotically optimal sampling-based planner.
G-RRT* rapidly identifies initial solutions and exploits the greedy informed set to enhance solution quality, particularly in high-dimensional spaces.
Throughout, we address geometric, holonomic planning with a path-length objective and do not consider kinodynamic planning.
In summary, we make the following contributions herein.
\begin{itemize}
    \item We provide a formal analysis of properties of the greedy informed set and consider how possible algorithmic choices influence planning performance.
    \item We present Greedy-RRT* (G-RRT*), a bi-directional, asymptotically optimal sampling-based planning algorithm that leverages the proposed greedy heuristic.\footnote{Information on the OMPL implementations of G-RRT* is publicly available at \url{https://github.com/utiasSTARS/ompl}.
    A native implementation of G-RRT* is also available in VAMP~\citep{thomason2024motions} at \url{https://github.com/KavrakiLab/vamp}.
    }
    \item We prove the completeness and asymptotic optimality of G-RRT*, building upon existing results from the sampling-based planning literature.
	\item We show experimentally that greedy exploitation in large planning domains improves success and convergence rates over state-of-the-art methods, using both simulations and manipulation datasets.
\end{itemize}

The article is organized as follows.
Section~\ref{sec:related_work} covers related literature, while Section~\ref{sec:preliminaries} describes our notation and problem formulation.
Section~\ref{sec:greedy_informed_set} defines the greedy informed set, while Section~\ref{sec:grrt_star} details the G-RRT* algorithm.
Section~\ref{sec:analysis} analyzes the properties of G-RRT*.
Section~\ref{sec:experiments} demonstrates the performance of G-RRT* in simulations and through manipulation experiments.
The article concludes with some discussion in Section~\ref{sec:conclusion}.

%% file: section2/related_work.tex
\section{Related Work}
\label{sec:related_work}

This section first reviews the existing literature on sampling-based planning algorithms (Section~\ref{sec:sampling_based_planning_review}), and then discusses methods for accelerating their convergence (Section~\ref{sec:sampling_review}).

\subsection{Sampling-Based Motion Planning}
\label{sec:sampling_based_planning_review}

Sampling-based motion planners can be broadly classified into multiple-query and single-query categories.
Multiple-query planners, such as probabilistic roadmaps~\citep{kavraki1996probabilistic,hsu1998finding}, construct a graph of collision-free paths that can be reused for different start-goal pairs. 
In contrast, single-query planners expand a tree toward randomly sampled states to solve individual planning problems; rapidly-exploring random trees~\citep{lavalle2001randomized} are a well-known example.
These sampling-based approaches are often effective in high-dimensional state spaces and reduce computation by avoiding explicit obstacle representations, unlike deterministic graph-search planners.
However, they typically yield only feasible (collision-free) paths and do not guarantee optimality.

\citet{karaman2011sampling} introduced PRM* and RRT*, optimal variants of PRM and RRT that guarantee asymptotic optimality.
RRT* expands a tree into free space through random sampling, but unlike RRT, it considers nearby vertices to select the best parent and incrementally rewires them to improve path quality.
This rewiring enables RRT* to converge asymptotically to the optimal solution, a key property in sampling-based motion planning.

The RRT$^{\#}$ algorithm~\citep{arslan2013use} extends the local rewiring of RRT* to the global level using dynamic programming, accelerating convergence by removing states that cannot improve the current solution.
\citet{karaman2011anytime} proposed a branch-and-bound pruning scheme that deletes vertices whose cost-to-come plus a lower bound on the cost-to-go exceeds the cost of the current best path.
This pruning step eliminates vertices that are unlikely to contribute to better solutions and enhances real-time performance.
Quick-RRT*~\citep{jeong2019quick} refines parent selection and rewiring by also considering ancestors of nearby vertices, up to a user-defined depth, as candidate parents.
Similarly, F-RRT*~\citep{liao2021f} enhances initial path quality and convergence by generating random parent vertices close to obstacles.

Various extensions have been proposed to accelerate the convergence of RRT*, including bi-directional variants \citep{kuffner2000rrt,klemm2015rrt} and methods that relax optimality to near-optimality \citep{dobson2014sparse,salzman2016asymptotically}.
Our proposed approach similarly employs bi-directional search to rapidly generate initial solutions.
However, because these methods sample uniformly across the entire state space, their convergence slows markedly in high-dimensional problems, increasing the computational effort required to find optimal solutions.

\subsection{State Space Sampling Methods}
\label{sec:sampling_review}

Prior work on sampling-based planning has emphasized ways to improve the sampling process.
While uniform sampling preserves global optimality, performance can often be improved by biasing samples toward regions more likely to yield better solutions.
For example, \citet{bialkowski2013free} proposed a biased sampling method that records past collisions to direct future samples away from obstacles.
Similarly, \citet{kim2014cloud} used a generalized Voronoi graph to decompose free space into spheres of varying radii, forming a dynamic sampling `cloud' that evolves as the best solution is refined.

In P-RRT* and PQ-RRT* \citep{qureshi2016potential,li2020pq}, random sampling is guided by artificial potential fields towards more promising state space regions~\citep{khatib1986real}, trading off exploration for exploitation.
Compared with traditional rejection sampling, these free-space-biased methods quickly find improved paths and reduce the number of rejected samples, but they continue to draw samples that do not improve the current solution.

Several extensions of RRT* have been developed to accelerate convergence.
\citet{akgun2011sampling} introduced a bi-directional version of RRT* that uses path-biased sampling and heuristic sample rejection for high-dimensional problems.
Related work~\citep{islam2012rrt,alterovitz2011rapidly,faroni2024adaptive} applies similar path-biasing and refinement ideas, achieving faster convergence but often favouring locally optimal solutions over globally better ones.
Techniques like rectangular rejection sampling~\citep{ferguson2006anytime,otte2013c} can improve convergence, but their effectiveness decreases as dimensionality increases.

To address this scalability issue, \citet{gammell2014informed,gammell2018informed} proposed \emph{informed sampling}, implemented in Informed RRT*.
Once a feasible path is found, the algorithm samples only within an $n$-dimensional hyperellipsoid---called the $L^2$ informed set---bounded by the current solution cost, which shrinks as better solutions are found.
This principle has since been incorporated into batch-wise heuristic search planners, including BIT*~\citep{gammell2020batch}, AIT* and EIT*~\citep{strub2022adaptively}, and the recent bidirectional BLIT*~\citep{wang2025asymptotically}.

Although informed sampling is highly effective, early feasible paths often include redundant states, producing tortuous trajectories with many unnecessary twists and turns.
These paths can yield overly large informed sets, which slow convergence, particularly in high-dimensional state spaces.
Consequently, later studies~\citep{kim2015informed,jiang2020informed,wilson2025aorrtc} integrated path simplification and anytime refinement to tighten the informed set and improve solution quality, achieving state-of-the-art results on the MotionBenchMaker dataset~\citep{chamzas2021motionbenchmaker}.
We pursue a different strategy: leveraging the greedy heuristic proposed in \citep{kyaw2022energy} within a bi-directional framework to prioritize sampling in regions most likely to yield improvement.
By coupling greedy state-space biasing with efficient tree growth, our method accelerates convergence without sacrificing global solution optimality.

%% file: section3/preliminaries.tex
\section{Preliminaries \label{sec:preliminaries}}

We begin by defining the notation used throughout the paper and introducing key planning concepts in Section~\ref{sec:notation}.
Section~\ref{sec:feasible_and_optimal} then reviews the formal definitions of the feasible and optimal path-planning problems, providing a foundation for our analysis of the G-RRT* algorithm.

\subsection{Notation}
\label{sec:notation}

Let the state space of a path planning problem be denoted by $\mathcal{X} \subseteq \mathbb{R}^n$, and let $\mathbf{x} \in \mathcal{X}$ denote a single state.
Let $\mathcal{X}_{\text{obs}} \subsetneq \mathcal{X}$ be the set of states in collision with obstacles, and let $\mathcal{X}_{\text{free}} = \mathcal{X} \setminus \mathcal{X}_{\text{obs}}$ be the set of collision-free states.
A path is a continuous function $\pi : [0,1] \to \mathcal{X}$, and we denote by $\Sigma$ the set of all such paths.
The initial and goal states are denoted by $\mathbf{x}_{I} \in \mathcal{X}_{\text{free}}$ and $\mathbf{x}_{G} \in \mathcal{X}_{\text{free}}$, respectively.
Each planning problem is associated with a cost function $c : \Sigma \to \mathbb{R}_{\ge 0}$ that maps a path to a non-negative real value.
Herein, we define the cost of a path as its length under the standard Euclidean metric on $\mathbb{R}^n$.

As is standard in sampling-based motion planning, we represent paths using graphs and, in particular, restrict our attention to trees, which are directed acyclic graphs.
A tree is incrementally constructed by connecting states through feasible transitions that respect collision constraints.
Let $\mathcal{T} = (V, E)$ denote such a tree, where $V$ is a set of vertices and $E \subseteq V \times V$ is a set of directed edges.
Each vertex corresponds to a state $\mathbf{x} \in \mathcal{X}$, and each edge $(\mathbf{x}_a, \mathbf{x}_b) \in E$ represents a valid (collision-free) transition from $\mathbf{x}_a$ to $\mathbf{x}_b$.

With a slight abuse of notation, let $c(\mathbf{x}_a, \mathbf{x}_b)$ denote the cost of the path between two vertices $\mathbf{x}_a, \mathbf{x}_b \in V$ in the tree $\mathcal{T}$.
An admissible heuristic estimate of this cost is $\hat{c} : V \times V \to \mathbb{R}_{\ge 0}$, which satisfies
\begin{equation}
\forall\ \mathbf{x}_a, \mathbf{x}_b \in V, \quad \hat{c}(\mathbf{x}_a, \mathbf{x}_b) \leq c(\mathbf{x}_a, \mathbf{x}_b).
\end{equation}
The cost-to-come is $g(\mathbf{x}) = c(\mathbf{x}_I, \mathbf{x})$, representing the true cost of the optimal path from the initial vertex $\mathbf{x}_I$ to vertex $\mathbf{x}$.
Its admissible heuristic estimate is $\hat{g}(\mathbf{x}) = \hat{c}(\mathbf{x}_I, \mathbf{x})$.
Similarly, the cost-to-go is $h(\mathbf{x}) = c(\mathbf{x}, \mathbf{x}_G)$, the true cost of the optimal path from vertex $\mathbf{x}$ to the goal vertex $\mathbf{x}_G$.
Its admissible heuristic estimate is $\hat{h}(\mathbf{x}) = \hat{c}(\mathbf{x}, \mathbf{x}_G)$.
The total cost of the optimal path through $\mathbf{x}$ is $f(\mathbf{x}) = g(\mathbf{x}) + h(\mathbf{x})$, and its admissible heuristic estimate is $\hat{f}(\mathbf{x}) = \hat{g}(\mathbf{x}) + \hat{h}(\mathbf{x})$.
Where unambiguous, the same symbol $g(\mathbf{x})$ also denotes the cost-to-come of $\mathbf{x}$ along its current path in a tree, which is an upper bound on this optimal value.

\subsection{Feasible and Optimal Path Planning}
\label{sec:feasible_and_optimal}

A planning problem instance is defined by the state space $\mathcal{X}_{\text{free}}$, an initial state $\mathbf{x}_I$, and a goal state $\mathbf{x}_G$, typically expressed as a tuple $\left(\mathcal{X}_{\text{free}}, \mathbf{x}_I, \mathbf{x}_G\right)$.
In general, two types of path planning problems may be considered: the \emph{feasible planning problem} and the \emph{optimal planning problem}.
We define these below.

\vspace{-0.5em}
\begin{definition}[Feasible path planning problem]
\label{def:feasible_planning}
Find a path $\pi^{\prime} \in \Sigma$ from $\mathbf{x}_I$ to $\mathbf{x}_G$ through collision-free space such that
\begin{equation}
\label{eqn:feasible_planning}
\begin{aligned}
    \pi^{\prime} \in \{\pi \in \Sigma \mid \pi(0)=\mathbf{x}_{I},\ \pi(1)=\mathbf{x}_{G}, \\ \forall s \in [0,1],\ \pi(s) \in \mathcal{X}_{\text{free}} \}.
\end{aligned}
\end{equation}
\end{definition}
\noindent There are usually many solutions to the feasible path planning problem.

\vspace{-0.5em}
\begin{definition}[Optimal path planning problem]
\label{def:optimal_planning}
Find a feasible path $\pi^*$ from $\mathbf{x}_I$ to $\mathbf{x}_G$ through collision-free space that minimizes the cost functional $c : \Sigma \to \mathbb{R}_{\geq 0}$:
\begin{equation}
\label{eqn:optimal_planning}
\begin{aligned}
    \pi^* = \underset{\pi \in \Sigma}{\operatorname{\arg\min}} \{c(\pi) \mid \pi(0)= \mathbf{x}_{I},\ \pi(1)=\mathbf{x}_{G}, \\[-2.3mm] \forall s \in [0,1],\ \pi(s) \in \set{\text{free}} \}.
\end{aligned}
\end{equation}
\end{definition}

A path planning algorithm is considered \emph{robustly feasible} if there exists a solution path with strong $\delta$-clearance, that is, remaining at least a distance $\delta$ from any obstacle in $\mathbb{R}^n$, for some $\delta > 0$~\citep{karaman2011sampling}.
A solution path $\pi^{*}$ is \emph{robustly optimal} if there exists another path $\pi_{0}$ in the same homotopy class, also with strong $\delta$-clearance, such that $c(\pi_0) = \min\{c(\pi) \mid \pi \text{ is feasible}\}$.
Notably, any formal guarantees provided by sampling-based algorithms are typically probabilistic in nature.
An algorithm is \emph{probabilistically complete} if the probability of finding a feasible path (when one exists) approaches one as the number of samples tends to infinity.
It is \emph{almost-surely asymptotically optimal} if the probability that the solution cost converges to the optimal cost approaches one as the number of samples tends to infinity.
Asymptotic optimality implies probabilistic completeness.
We refer interested readers to \cite{karaman2011sampling} for further details.

%% file: section4/greedy_informed_set.tex
\section{The Greedy Informed Set}
\label{sec:greedy_informed_set}

Informed sampling-based planners focus the search on promising regions of the state space to find better paths once an initial solution has been found.
These regions correspond to the \emph{omniscient set} $\set{f}$, the collection of states that could yield a better solution, which is generally unknown because computing it requires solving the problem exactly~\citep{gammell2018informed}.
Heuristic-based estimates, such as the informed set $\infset$, are commonly used to approximate $\set{f}$.
The informed set is an admissible over-approximation of the omniscient set, defining a subset of $\set_{\text{free}}$ where samples are likely to improve the current solution, provided the heuristic never overestimates the true cost.
However, in problems with many homotopy classes or in high-dimensional spaces, the heuristic used in $\infset$ often provides very little information, resulting in an overly large informed set.
Consequently, the probability of drawing a random sample from $\infset$ that also belongs to $\set{f}$ is low.

The \emph{greedy informed set}, introduced by~\citet{kyaw2022energy}, addresses this limitation by eliminating direct dependence on the current solution cost.
Instead, it constructs a smaller hyperellipsoid from the states along the current solution path, which is bounded above by the admissible informed set.
This section formally defines the greedy informed set and examines how heuristic exploitation in greedy informed sampling influences optimality.
We consider only holonomic planning problems in $\mathbb{R}^n$ with a path-length objective under the Euclidean norm (i.e., where all degrees of freedom are directly controllable and the cost depends solely on geometric path length).
For any state $\mathbf{x} \in \mathbb{R}^n$, let $\hat{f}(\mathbf{x})$ denote the standard $L^2$ informed heuristic, which provides an admissible lower bound on the cost of any path from the initial state $\mathbf{x}_I$ to the goal state $\mathbf{x}_G$ passing through $\mathbf{x}$,
\begin{equation}
\label{eqn:l2-heuristic}
\Hat{f}(\mathbf{x}) = \Norm{\mathbf{x} - \mathbf{x}_I}_2 + \Norm{\mathbf{x} - \mathbf{x}_G}_2.
\end{equation}

\begin{definition}[$L^2$ greedy informed set]\label{def:greedy_informed_set}
Let $\mathbf{x}_{\text{max}}$ denote the state along the current solution path $\pi$ with the maximum admissible heuristic cost, defined by $\Hat{f}(\cdot)$ in~(\ref{eqn:l2-heuristic}),
\begin{equation}
    \mathbf{x}_{\text{max}} \coloneqq 
    \underset{\mathbf{x} \in \pi}{\operatorname{\arg\max}} 
    \left\{ \Hat{f}(\mathbf{x}) \right\}.
\end{equation}
The $L^2$ \emph{greedy informed set}, $\smash{\ginfset}$, is the subset of collision-free states, $\mathcal{X}_{\text{free}}$, consisting of all $\mathbf{x}$ whose heuristic costs are less than or equal to that of the greedily chosen state $\mathbf{x}_{\text{max}}$ along the current solution path (Figure~\ref{fig:greedy_informed_set}),
\begin{equation}
\label{eqn:greedy_informed_set}
    \ginfset = 
    \left\{
    \mathbf{x} \in \mathcal{X}_{\text{free}}
    \mid
    \Hat{f}(\mathbf{x}) \leq \Hat{f}(\mathbf{x}_{\text{max}})
    \right\}.
\end{equation}
\end{definition}

\begin{figure}[!t] 
\centering
\includegraphics[width=0.85\linewidth]{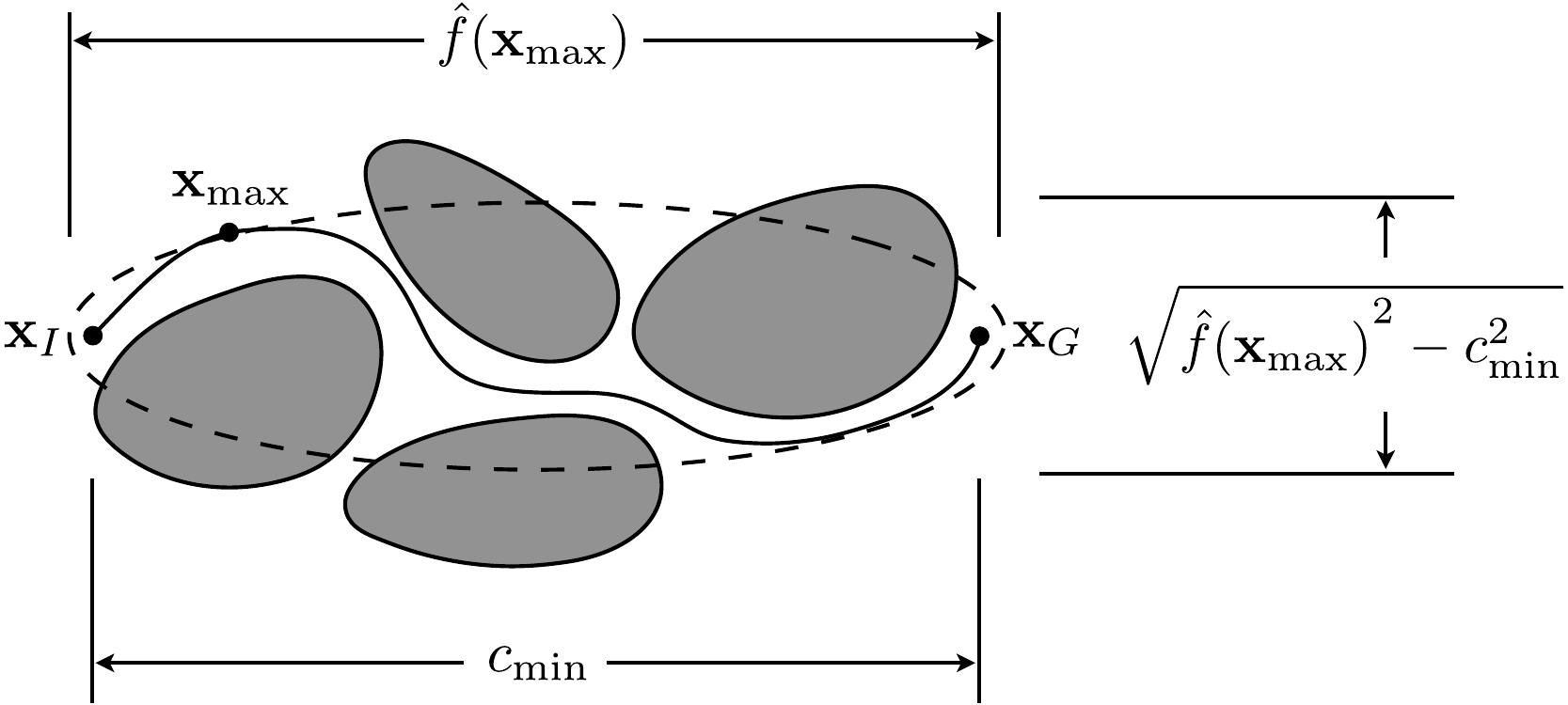}
\caption{
Illustration of the $L^2$ greedy informed set $\smash{\ginfset}$ in a planning problem with a path-length objective in $\mathbb{R}^2$.
The greedy subset (dashed ellipse) is defined by the hypothetical minimum cost $c_{\text{min}}$ from the initial state $\mathbf{x}_I$ to the goal state $\mathbf{x}_G$, and by the heuristic cost of the state along the solution path with the highest admissible value $\smash{\Hat{f}(\mathbf{x}_{\text{max}})}$, used as a transverse diameter.
}
\label{fig:greedy_informed_set}
\end{figure}

\begin{proposition}[Inclusion of the greedy informed set within the informed set]
\label{prop:inclusion}
Let $\pi$ be the current solution path from $\mathbf{x}_I$ to $\mathbf{x}_G$ with cost $c_i = c(\pi)$, and let
$\smash{\mathcal{X}_{\Hat{f}} = \{\mathbf{x} \in \mathcal{X}_{\mathrm{free}} \mid \Hat{f}(\mathbf{x}) \le c_i\}}$ be the corresponding informed set, with $\Hat{f}(\cdot)$ as defined in \eqref{eqn:l2-heuristic}.
Then the greedy informed set satisfies
\begin{equation*}
\mathcal{X}_{\mathrm{greedy}} \subseteq \mathcal{X}_{\Hat{f}}.
\end{equation*}
\end{proposition}
\begin{proof}
By definition, $\Hat{f}(\mathbf{x})$ is an admissible heuristic, meaning it provides a lower bound on the cost of any path from $\mathbf{x}_I$ to $\mathbf{x}_G$ passing through $\mathbf{x}$.
Because the current solution $\pi$ is exactly such a path for any state $\mathbf{x} \in \pi$, admissibility guarantees that $\Hat{f}(\mathbf{x}) \leq c(\pi) = c_i$.
Since this holds for all states along $\pi$, it also holds specifically for the state with the maximum heuristic cost, ensuring $\Hat{f}(\mathbf{x}_{\text{max}}) \leq c_i$.
By Definition~\ref{def:greedy_informed_set}, any state $\mathbf{x} \in \smash{\ginfset}$ satisfies $\Hat{f}(\mathbf{x}) \leq \Hat{f}(\mathbf{x}_{\text{max}})$, which implies $\Hat{f}(\mathbf{x}) \leq c_i$.
Since $\mathbf{x} \in \mathcal{X}_{\mathrm{free}}$ and its heuristic cost is bounded by $c_i$, $\mathbf{x}$ is contained within the informed set $\smash{\mathcal{X}_{\Hat{f}}}$.
Thus, $\smash{\mathcal{X}_{\mathrm{greedy}} \subseteq \mathcal{X}_{\Hat{f}}}$.
\end{proof}

Because the greedy informed set is bounded above by the informed set (Proposition~\ref{prop:inclusion}), it yields smaller ellipsoidal regions than the original informed set, especially in high-dimensional problems with complex solution paths.
Consequently, sampling from this smaller subset increases the likelihood of finding states that also belong to the omniscient set.
However, because the greedy informed set is constructed using knowledge from the current solution path---similar to path-biasing methods---an algorithm relying on it may in some cases bias the search toward locally optimal solutions, as the set might exclude portions of the omniscient set.
Therefore, continued exploration of other homotopy classes, that is, sampling from the informed set, remains essential for planners utilizing the greedy informed set to preserve asymptotic optimality guarantees.

Sampling from informed sets is a necessary condition for asymptotic optimality in sampling-based motion planning, since states along the optimal path always lie within these sets~\citep{gammell2018informed}.
The greedy informed set may also contain states from the optimal path, but only if the current solution path and the true optimal path belong to the same homotopy class.
In the following, we draw on basic concepts from homotopy theory~\citep{hatcher2002algebraic} to establish the conditions under which an algorithm sampling from the greedy informed set can still converge to the optimal solution (Theorem~\ref{theorem:optimality}), and to discuss cases where this property may fail to hold (Remark~\ref{remark:non_optimality}).

\begin{theorem}[Inclusion of the optimal path in the greedy informed set]
\label{theorem:optimality}
Let $\Hat{f}(\cdot)$ be the $L^2$ heuristic defined in~(\ref{eqn:l2-heuristic}).
Let $\pi$ be the current solution path and $\pi^*$ the optimal path between the same initial and goal states $\mathbf{x}_I$ and $\mathbf{x}_G$.
If $\pi$ and $\pi^*$ lie within the same homotopy class, then every state $\mathbf{x}^* \in \pi^*$ satisfies
\begin{equation}
\hat{f}(\mathbf{x}^*) \leq \hat{f}(\mathbf{x}_{\mathrm{max}}).
\end{equation}
Consequently, $\pi^* \subseteq \ginfset$.
\end{theorem}
    
\begin{figure}[!tb] 
\centering
\includegraphics[width=0.65\linewidth]{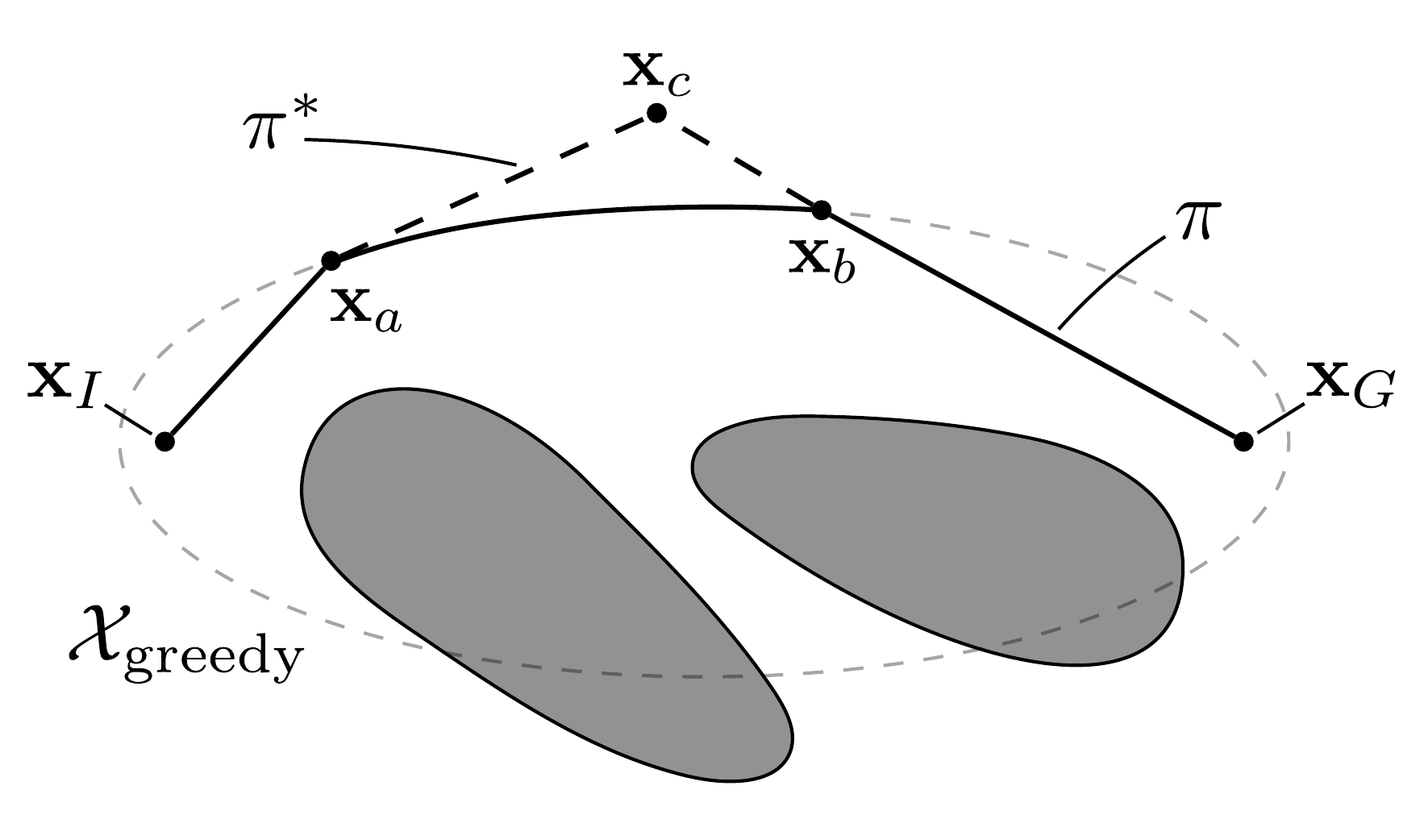}
\caption{
Illustration of the current solution path $\pi$ and another path $\pi^*$ within the same homotopy class, sharing the same initial and goal states, $\mathbf{x}_I$ and $\mathbf{x}_G$.
The greedy informed set $\ginfset$ (dashed gray ellipse) is constructed using the maximum heuristic cost along $\pi$.
The path $\pi^*$ intersects the boundary of $\ginfset$ at $\mathbf{x}_a$ and $\mathbf{x}_b$ and passes through a state $\mathbf{x}_c$ that lies outside this set.
}
\label{fig:theorem1}
\end{figure}

\begin{proof}
Consider the example shown in Figure~\ref{fig:theorem1}.
Let $\pi$ denote the current solution path in $\mathcal{X}_{\text{free}}$ from $\mathbf{x}_I$ to $\mathbf{x}_G$, and let $\mathbf{x}_a$ and $\mathbf{x}_b$ be two states along $\pi$ with the same maximum heuristic cost, 
$\Hat{f}(\mathbf{x}_a) = \Hat{f}(\mathbf{x}_b) = \Hat{f}(\mathbf{x}_{\text{max}})$.
Let the path $\pi$ be constructed such that every state $\mathbf{x'}$ on the subpath between $\mathbf{x}_{a}$ and $\mathbf{x}_{b}$ also satisfies $\Hat{f}(\mathbf{x'}) = \Hat{f}(\mathbf{x}_{\text{max}})$:
\begin{equation*}
\forall\,\mathbf{x'} \in \mathcal{Z}, \, 
\Hat{f}(\mathbf{x'}) = \Hat{f}(\mathbf{x}_{\text{max}}),
\end{equation*}
where $\mathcal{Z} := \{\mathbf{x}' \in \pi_{ab} \mid \pi_{ab}(0) = \mathbf{x}_a,\, \pi_{ab}(1) = \mathbf{x}_b,\ \pi_{ab} \subset \pi\}$.
By Definition~\ref{def:greedy_informed_set}, $\ginfset$ is constructed using $\mathbf{x}_{\text{max}}$ of $\pi$.
We show that the optimal path $\pi^*$ also lies within $\ginfset$ if $\pi$ and $\pi^*$ belong to the same homotopy class.

Suppose that $\pi^*$ contains a state $\mathbf{x}_c$ not in $\ginfset$,
\begin{equation}
\label{eqn:thm1-not-inclusion}
\exists\, \mathbf{x}_c \in \pi^* \ \text{s.t.}\ \mathbf{x}_c \notin \ginfset.
\end{equation}
Since $\mathbf{x}_I, \mathbf{x}_G \in \ginfset$, $\pi^*$ must exit the hyperellipsoid at some point and later re-enter it.
Consequently, there exist at least two intersection points $\mathbf{x}_a$ and $\mathbf{x}_b$ between $\pi^*$ and the boundary of $\ginfset$.

Because $\pi$ and $\pi^*$ are homotopic, they can be continuously deformed into one another without intersecting obstacles while keeping endpoints fixed~\citep{hatcher2002algebraic}.
The same holds for their subpaths $\pi_{ab}$ and $\pi^*_{abc}$.
The cost of the direct (geodesic) subpath between $\mathbf{x}_a$ and $\mathbf{x}_b$ is strictly less than that of the detour passing through $\mathbf{x}_c$:
\begin{equation}
\label{eqn:thm1-cost-compare}
    c(\pi_{ab}) < c(\pi^*_{abc}).
\end{equation}
By definition, every subpath of an optimal path must also be optimal~\citep{cormen2022introduction}; 
hence, (\ref{eqn:thm1-cost-compare}) contradicts the optimality of $\pi^*$.
Therefore, the assumption in~(\ref{eqn:thm1-not-inclusion}) is false, and $\pi^* \subseteq \ginfset$.
\end{proof}

While Theorem~\ref{theorem:optimality} establishes that the optimal path $\pi^*$ lies within the greedy informed set $\ginfset$ when it shares the same homotopy class as the current solution path $\pi$, this guarantee no longer holds when the two paths are homotopically distinct.
In such cases, $\ginfset$ may fail to include $\pi^*$, motivating an examination of scenarios where this exclusion occurs.

\begin{remark}[Non-inclusion of the optimal path in the greedy informed set]
\label{remark:non_optimality}
The greedy informed set $\ginfset$ does not necessarily contain the optimal path $\pi^*$ if the current solution path and the optimal path lie in different homotopy classes.
\end{remark}

Remark~\ref{remark:non_optimality} follows directly from Theorem~\ref{theorem:optimality} and can be illustrated with a simple counterexample.
Consider a planning problem in a maze-like environment with many homotopy classes between the initial and goal states (Figure~\ref{fig:analysis}).
Let $\pi$ be a feasible path that passes through a narrow corridor and requires several sharp turns to navigate around obstacles.
In contrast, let $\pi^*$ denote an optimal path in a different homotopy class that avoids the corridor by circumventing the obstacles entirely.
By construction, the greedy informed set $\ginfset$ is concentrated around the current solution path $\pi$, primarily near the narrow passage.

Since $\pi$ and $\pi^*$ belong to different homotopy classes, the path $\pi$ cannot be continuously deformed into $\pi^*$ without intersecting obstacles.
Theorem~\ref{theorem:optimality} establishes that if $\pi$ and $\pi^*$ are homotopic, then $\pi^*$ lies in $\ginfset$.
However, in this case, $\ginfset$ is restricted to the vicinity of the narrow passage and therefore does not contain $\pi^*$, which circumvents the obstacles.
Thus, in environments of this type, where the current feasible path and the optimal path belong to different homotopy classes, the greedy informed set $\ginfset$ may fail to include the optimal path.

\begin{figure}[!tb] 
\centering
\includegraphics[width=0.8\linewidth]{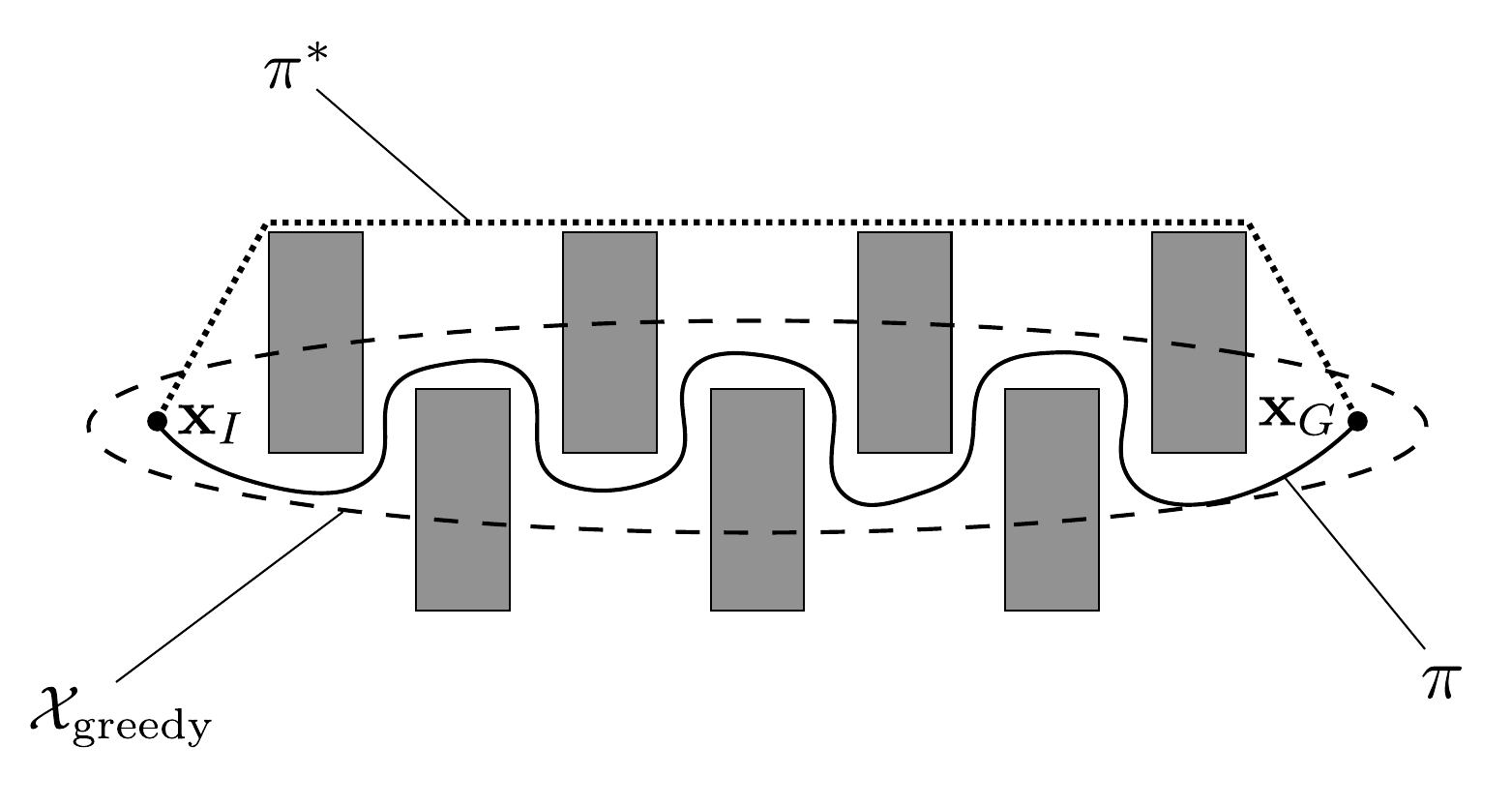}
\caption{
Illustration of the suboptimality of the greedy informed set in an example planning scenario.
The greedy informed set $\ginfset$, constructed from the current tortuous solution path $\pi$, is shown as a dashed ellipse and the optimal path $\pi^{*}$ as a dotted line.
In this case, $\ginfset$ fails to include some states that could improve the current solution cost (i.e., those leading towards $\pi^{*}$ and lying outside of the hyperellipsoid) due to the nature of its greedy exploitation.
}
\label{fig:analysis}
\end{figure}

This counterexample highlights an important limitation of relying solely on $\ginfset$.
When the optimal path lies in a different homotopy class from the current solution path, sampling exclusively from $\ginfset$ can lead to suboptimal performance.
Understanding this behaviour is crucial when choosing how frequently the planner should sample from $\ginfset$, as the impact of this choice depends on the distribution of homotopy classes and the type of planning problem being addressed.
A practical way to mitigate this issue is to also sample from the informed set with some probability (Section~\ref{sec:grrt_star}), thereby balancing exploration and exploitation within the planner.

%% file: section5/grrt_star.tex
\section{Greedy RRT* (G-RRT*)}
\label{sec:grrt_star}

\begin{figure*}[!t]
   \centering
   \begin{subfigure}[b]{0.24\textwidth}
       \centering
       \includegraphics[width=\textwidth]{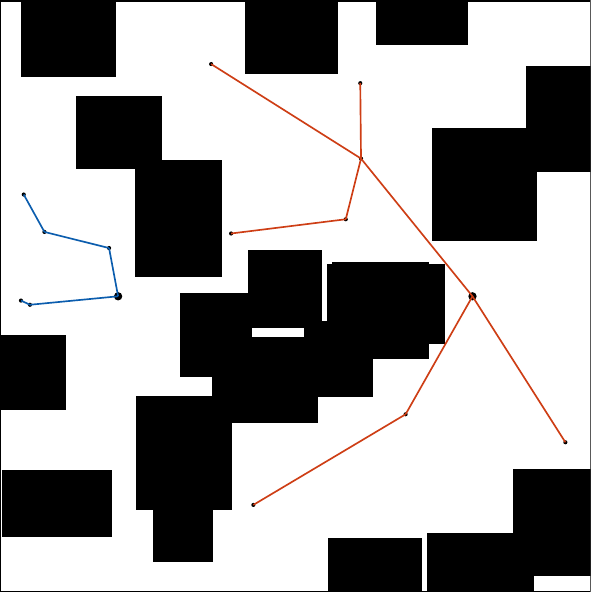}
       \captionsetup{justification=centering}
       \caption{}
   \end{subfigure}
   \begin{subfigure}[b]{0.24\textwidth}
       \centering
       \includegraphics[width=\textwidth]{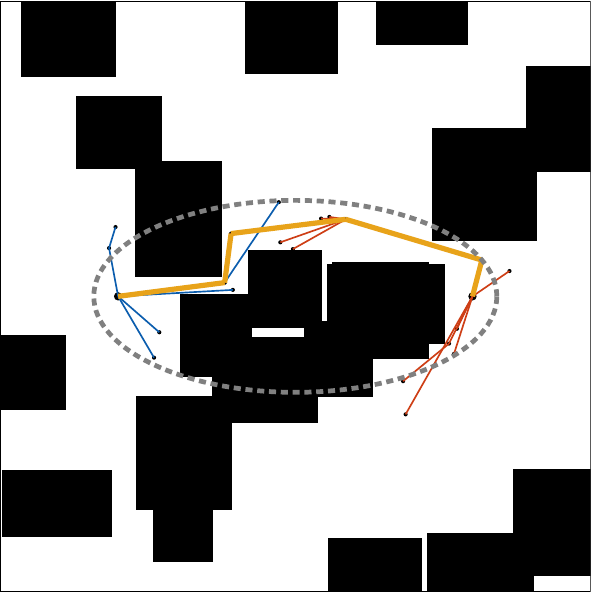}
       \captionsetup{justification=centering}
       \caption{}
   \end{subfigure}
   \begin{subfigure}[b]{0.24\textwidth}
       \centering
       \includegraphics[width=\textwidth]{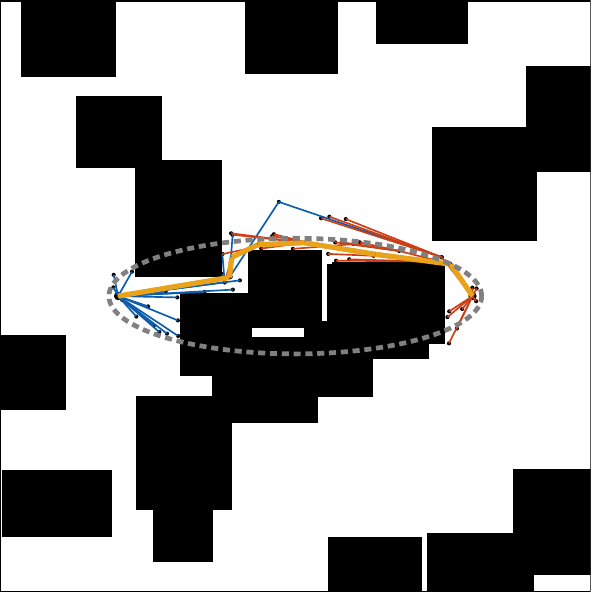}
       \captionsetup{justification=centering}
       \caption{}
   \end{subfigure}
   \begin{subfigure}[b]{0.24\textwidth}
       \centering
       \includegraphics[width=\textwidth]{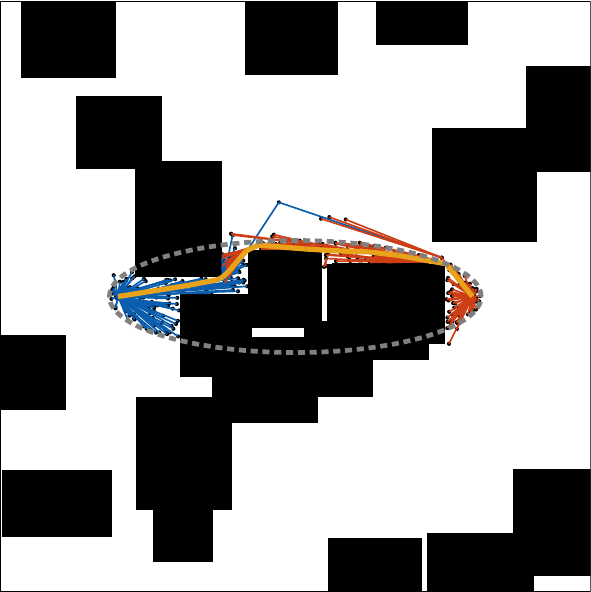}
       \captionsetup{justification=centering}
       \caption{}
   \end{subfigure}
\caption{Illustrations of the progress of bi-directional sampling-based search performed by the G-RRT* algorithm.
The initial and goal states are shown as large black dots, sampled states as small black dots, and the start and goal trees in blue and orange, respectively.
The current solution path is highlighted in yellow, and the $L^2$ greedy informed set---the set of states that could yield better solutions---is shown with gray dashed lines.
G-RRT* grows two trees rooted at the start and goal (a), where expansion is guided by a greedy connection heuristic that guides the two trees towards each other, producing an initial solution (b).
Subsequent sampling is then focused within the greedy informed set to incrementally refine the path (c--d), almost-surely asymptotically converging to the optimal solution.
}
\label{fig:g_rrtstar_progress}
\vspace{-2mm}
\end{figure*}

Greedy RRT* is an almost surely asymptotically optimal path planning algorithm that builds on RRT* and its bi-directional variants~\citep{klemm2015rrt, mashayekhi2020informed}.
It maintains two rapidly growing trees---one rooted at the start and the other at the goal---to explore the state space, and employs a greedy connection heuristic to guide them toward each other, similar to RRT-Connect~\citep{kuffner2000rrt}.
As an RRT*-like algorithm, it incrementally rewires the growing trees to preserve asymptotic optimality.
However, rather than sampling randomly throughout the state space, G-RRT* focuses the search on promising regions, once an initial solution is found.
Specifically, it employs a greedy version of the informed set introduced in Section~\ref{sec:greedy_informed_set} to exploit knowledge of the existing solution path and focus subsequent sampling (Figure~\ref{fig:g_rrtstar_progress}).
The complete algorithm is presented in Algorithm blocks~\ref{algo:greedy_rrtstar}--\ref{algo:connect_star}, with modifications to the bi-directional versions of RRT and RRT* highlighted in red.

\begin{algorithm}[!t]
\caption{Greedy RRT* $(\mathbf{x}_I, \mathbf{x}_G)$}\label{algo:greedy_rrtstar}
$V_a \gets \{ \mathbf{x}_I \}; E_a \gets \emptyset; \mathcal{T}_a=(V_a,E_a)$\;
$V_b \gets \{ \mathbf{x}_G \}; E_b \gets \emptyset; \mathcal{T}_b=(V_b,E_b)$\;
$E_{\text{sol'n}} \gets \emptyset$\;
{\color{red(ncs)}$c_{\text{best}} \gets \infty; \; c_{\text{max}} \gets \infty$\;}
\For {$i=1, \dots, n$}{
{\color{red(ncs)}$c_{\text{i}} \gets \text{ComputeBestCost}\,(E_{\text{sol'n}})$\label{algo1:line5}\;}
$\mathbf{x}_{\text{rand}} \gets \text{Sample}\,(\mathbf{x}_I,\mathbf{x}_G,c_{\text{i}})$\label{algo1:line6}\;
$\text{status}, \mathbf{x}_{\text{a}} \gets \text{Extend*}\,(\mathcal{T}_a=(V_a, E_a), \mathbf{x}_{\text{rand}})$\label{algo1:line7}\;
    
\If{$\textup{status} \neq \textup{TRAPPED}$\label{algo1:line8}}{$\text{status}, \mathbf{x}_{\text{b}} \gets \textup{Connect*}\,(\mathcal{T}_b=(V_b,E_b), \mathbf{x}_{\text{a}})$\label{algo1:line9}\;
        \If{$\textup{status} = \textup{REACHED}$\label{algo1:line10}}{
            $E_{\text{sol'n}} \gets E_{\text{sol'n}} \cup \{ \mathbf{x}_{\text{a}}, \mathbf{x}_{\text{b}} \}$\label{algo1:line11}\;
        }
    }
    
    $\textup{Swap}\,(\mathcal{T}_a=(V_a,E_a), \mathcal{T}_b=(V_b,E_b))$\label{algo1:line12}\;
}
\Return{$\mathcal{T}=(V_a \cup V_b, E_a \cup E_b)$}\;
\end{algorithm}

\begin{algorithm}[!t]
\caption{ComputeBestCost\\$(E_{\text{sol'n}} = \{(\mathbf{x}_a,\mathbf{x}_b) \;|\; \mathbf{x}_a \in V_a, \; \mathbf{x}_b \in V_b \} )$}\label{algo:compute_best_cost}
{\color{red(ncs)}
   $c_{\text{i}} \gets \infty$\label{algo2:line1}\;

   \If{$|E_{\text{sol'n}}| > 0$\label{algo2:line2}}{

    $c_{\text{i}} \gets \underset{(\mathbf{x}_a,\mathbf{x}_b) \in E_{\text{sol'n}}}{\operatorname{min}}
    \left\{ g_a(\mathbf{x}_a) + c(\mathbf{x}_a,\mathbf{x}_b) + g_b(\mathbf{x}_b) \right\}$\label{algo2:line3}\;

    \If{$\epsilon > \mathcal{U}([0,1])$\label{algo2:line4}}{

        \If{$c_{\mathrm{i}} < c_{\mathrm{best}}$\label{algo2:line5}}{
            $c_{\mathrm{best}} \gets c_{\text{i}}$\label{algo2:line6}\;
            $c_{\mathrm{max}} \gets
            \underset{
                \substack{
                    (\mathbf{x}_a,\mathbf{x}_b)\in E_{\text{sol'n}}\\
                    \mathbf{x}\in\{\mathbf{x}_a,\mathbf{x}_b\}\\
                    k\ge 0
                }}{\operatorname{max}} \Hat{f}\left(\mathrm{Parent}^k(\mathbf{x})\right)$\;
        }
        \Return $c_\mathrm{max}$\;
    }
}
\Return $c_{\mathrm{i}}$\;
}
\end{algorithm}

\begin{algorithm}[!t]
\caption{Sample $(\mathbf{x}_I, \mathbf{x}_G, c_{\text{max}})$}\label{algo:sample}
\Repeat{$\mathbf{x}_{\text{rand}}$ satisfies bounds}{\label{algo3:line1}
    \eIf{$c_{\text{max}} < \infty$\label{algo3:line2}}{
        $\mathbf{x}_{\text{rand}} \gets \text{SampleHyperEllipsoid}\,(\mathbf{x}_I, \mathbf{x}_G, c_{\text{max}})$\;
        \If{$\mathbf{x}_{\text{rand}} \in \mathcal{X} \cap \mathcal{X}_{\text{PHS}}$}{
            \Return{$\mathbf{x}_{\text{rand}}$}\label{algo3:line5}\;
        }
    }{\label{algo3:line6}
        $\mathbf{x}_{\text{rand}} \gets \text{SampleUniform}\,(\mathcal{X})$\;
        \Return{$\mathbf{x}_{\text{rand}}$}\label{algo3:line8}\;
    }
}
\end{algorithm}

\begin{algorithm}[!hptb]
\caption{Extend* $(\mathcal{T}=(V,E), \mathbf{x})$}\label{algo:extend_star}

    $\mathbf{x}_{\text{nearest}} \gets \text{Nearest}\,(\mathcal{T}=(V,E), \mathbf{x})$\label{algo4:line1}\;
    $\mathbf{x}_{\text{new}} \gets \text{Steer}\,(\mathbf{x}_{\text{nearest}}, \mathbf{x})$\label{algo4:line2}\;

    \If{\color{red(ncs)}$g(\mathbf{x}_{\text{nearest}})\! +\! c(\mathbf{x}_{\text{nearest}}, \mathbf{x}_{\text{new}})\! +\! \Hat{h}(\mathbf{x}_{\text{new}}) < c_{i}$}
    {
    \If {$\text{ObstacleFree}\,(\mathbf{x}_{\text{nearest}}, \mathbf{x}_{\text{new}})$\label{algo4:line3}}{
        $V \gets V \cup \{ \mathbf{x}_{\text{new}} \}$\label{algo4:line4}\;
        $X_{\text{near}} \gets \text{Near}\,(\mathcal{T}=(V,E), \mathbf{x}_{\text{new}}, r_{\text{rewire}})$\label{algo4:line5}\;
        $\mathbf{x}_{\text{min}} \gets \mathbf{x}_{\text{nearest}}$\label{algo4:line7}\;
        \ForEach{$\mathbf{x}_{\text{near}} \in X_{\text{near}}$}{
            $c_{\text{near}} \gets g(\mathbf{x}_{\text{near}}) + c(\mathbf{x}_{\text{near}}, \mathbf{x}_{\text{new}})$\;
            \If{$c_{\text{near}} < g(\mathbf{x}_{\text{min}}) + c(\mathbf{x}_{\text{min}}, \mathbf{x}_{\text{new}})$}{
                \If{$\text{ObstacleFree}\,(\mathbf{x}_{\text{near}}, \mathbf{x}_{\text{new}})$}{
                    $\mathbf{x}_{\text{min}} \gets \mathbf{x}_{\text{near}}$\;
                }
            }
        }
        $E \gets E \cup \{ \mathbf{x}_{\text{min}}, \mathbf{x}_{\text{new}} \}$\label{algo4:line13}\;
        \ForEach{$\mathbf{x}_{\text{near}} \in X_{\text{near}}$\label{algo4:line14}}{
            \If{$g(\mathbf{x}_{\text{new}}) + c(\mathbf{x}_{\text{new}}, \mathbf{x}_{\text{near}}) < g(\mathbf{x}_{\text{near}})$}{
                \If{$\text{ObstacleFree}\,(\mathbf{x}_{\text{new}}, \mathbf{x}_{\text{near}})$}{
                    $\mathbf{x}_{\text{parent}} \gets \text{Parent}\,(\mathbf{x}_{\text{near}})$\;
                    $E \gets E \setminus \{(\mathbf{x}_{\text{parent}},\mathbf{x}_{\text{near}})\}$\;
                    $E \gets E \cup \{(\mathbf{x}_{\text{new}},\mathbf{x}_{\text{near}})\}$\label{algo4:line19}\;
                }
            }
        }
        \eIf{$\mathbf{x}_{\text{new}} = \mathbf{x}$\label{algo4:line20}}
        {
            \Return{$\textup{REACHED}, \mathbf{x}_{\text{new}}$}\label{algo4:line21}\;
        }
        {\label{algo4:line22}
            \Return{$\textup{ADVANCED}, \mathbf{x}_{\text{new}}$}\label{algo4:line23}\;
        }
    }
    }
    \Return{$\textup{TRAPPED}, \mathbf{x}_{\text{new}}$}\label{algo4:line24}\;
\end{algorithm}

\begin{algorithm}[!hptb]
\caption{Connect* $(\mathcal{T}=(V,E), \mathbf{x})$}\label{algo:connect_star}

\Repeat{$\textup{status} \neq \textup{ADVANCED}$}
{
    $\text{status}, \mathbf{x}_{\text{new}} \gets \text{Extend*}\,(\mathcal{T}=(V,E), \mathbf{x})$\;
}

\end{algorithm}

\subsection{Greedy Informed Sampling}
\label{sec:greedy_informed_sampling}

G-RRT* finds better paths by simultaneously growing two trees, $\mathcal{T}_a = (V_a, E_a)$ from the start and $\mathcal{T}_b = (V_b, E_b)$ from the goal, consisting of the vertices $V_a \cup V_b$ and edges $E_a \cup E_b$, each expanding toward randomly sampled states in the free space.
Each tree incrementally rewires nearby vertices to minimize their cost-to-come.
Once an initial solution is found, G-RRT* exploits heuristic information from the current solution to bias sampling toward a progressively shrinking hyperellipsoidal region called the greedy informed set.
It also samples from the informed subset to maintain exploration and ensure asymptotic optimality.
Specifically, the balance between exploration and exploitation is controlled by a parameter $\epsilon \in [0,1]$, called the \emph{greedy biasing ratio}.
With probability $\epsilon$, the algorithm samples from the greedy informed set to exploit the current best path; with probability $1 - \epsilon$, it samples from the broader informed subset to encourage exploration and retain asymptotic optimality.
Thus, finding a good balance between uniform sampling for exploration and path-biased sampling for exploitation is necessary for G-RRT* to rapidly reduce the search space and achieve faster convergence toward globally optimal solutions.

The cost bound that defines this sampling region in each iteration is computed by Algorithm~\ref{algo:compute_best_cost}.
When no solution has been found, the procedure returns an infinite cost, and G-RRT* samples uniformly over the state space.
Once the two trees are connected, it finds the lowest-cost connection in the solution connection set $E_{\mathrm{sol'n}}$ and sets $c_{\mathrm{i}}$ to the corresponding solution cost.
Following the greedy biasing ratio above, the procedure compares $\epsilon$ against a uniform sample $\mathcal{U}([0,1])$ on the unit interval, returning the greedy cost bound $c_{\mathrm{max}}$ with probability $\epsilon$ and the current solution cost $c_{\mathrm{i}}$ with probability $1-\epsilon$.
The states along the current best path are recovered by following parent pointers from each connecting vertex back to its tree root, where $\mathrm{Parent}^k(\mathbf{x})$ denotes the $k$-th ancestor of $\mathbf{x}$ and $\mathrm{Parent}^0(\mathbf{x}) = \mathbf{x}$.
The cost of the best solution found so far, $c_{\mathrm{best}}$, and the greedy cost bound, $c_{\mathrm{max}}$, persist across iterations.
Algorithm~\ref{algo:greedy_rrtstar} initializes both to infinity before the search begins.
Since $c_{\mathrm{max}}$ depends only on the current best path, G-RRT* recomputes it only when a lower-cost solution is found and reuses the cached value otherwise, keeping the cost of maintaining the greedy informed set negligible relative to the overall planning time.
G-RRT* adds only a constant-time heuristic gating check before the collision-checking and rewiring operations, so its per-iteration cost matches that of the existing algorithms up to this constant.

\subsection{Implementation Details}
\label{sec:implementation_details}

G-RRT* employs a balanced search strategy~\citep{kuffner2005efficient} to keep both trees approximately equal in size while maintaining their rapidly exploring behaviour.
Since collision checking is computationally expensive in practice,  G-RRT* also uses admissible heuristics to gate the vertex extension process (highlighted as red in Algorithm~\ref{algo:extend_star}).
An edge from the nearest vertex to a new vertex is validated only if it can improve the current best solution cost.
This gating step also helps G-RRT* to reduce overlap between the two trees, limiting the number of added vertices.
G-RRT* prunes the trees as in \cite{gammell2018informed}, removing vertices whose heuristic values exceed the greedy best heuristic cost.
Additionally, because obtaining an initial solution as quickly as possible is important, best parent selection and rewiring is delayed until an initial solution is found.
The remaining planning time is then used to rewire to improve solution quality.

%% file: section6/analysis.tex
\section{Analysis}
\label{sec:analysis}

This section analyzes the theoretical guarantees of G-RRT*, first establishing its probabilistic completeness in Section~\ref{sec:probabilistic_completeness} and then proving asymptotic optimality in Section~\ref{sec:asymptotic_optimality}.

\subsection{Probabilistic Completeness}
\label{sec:probabilistic_completeness}

Probabilistic completeness requires that an algorithm find an initial feasible solution, when one exists, as the number of samples tends to infinity.
Since \mbox{G-RRT*} uses the same tree extension and connection strategy as RRT-Connect, its probabilistic completeness follows directly from the arguments established in Theorem~1 of~\citet{kuffner2000rrt}, with a rigorous treatment of the geometric case given by~\cite{kleinbort2018probabilistic}.
While probabilistic completeness can be sensitive to implementation details, the known failure cases involve the vertex extension mechanism under differential constraints \citep{kunz2015kinodynamic} and do not arise in the holonomic geometric setting considered here.
Because G-RRT* modifies RRT-Connect with several components (heuristic gating, greedy informed sampling, delayed rewiring, and tree pruning), we verify that none of them affects the conditions required for probabilistic completeness.

Each of these components remains inactive until an initial solution has been found.
Before the first solution, the best solution cost is infinite ($c_{\mathrm{i}} = \infty$ in Algorithm~\ref{algo:compute_best_cost}), so the gating condition in Algorithm~\ref{algo:extend_star} holds trivially, sampling defaults to a uniform distribution over $\mathcal{X}$ (Algorithm~\ref{algo:sample}), and neither pruning nor rewiring is enabled.
While searching for the initial solution, G-RRT* therefore performs the same extension and connection operations as RRT-Connect.
Since uniform sampling assigns positive probability to every region of $\set{\mathrm{free}}$ with positive measure and the geometric extension step is unchanged, the vertex sets $V_a$ and $V_b$ become dense in $\set{\mathrm{free}}$ as the number of iterations $k \to \infty$.
Because the greedy connection heuristic generates all the standard RRT vertices, along with additional ones, it aids in covering $\set{\mathrm{free}}$ and therefore does not affect the completeness guarantee.
G-RRT* is therefore probabilistically complete, and the components above take effect only after an initial solution exists, where they influence the refinement of the solution rather than its discovery (Section~\ref{sec:asymptotic_optimality}).

\subsection{Asymptotic Optimality}
\label{sec:asymptotic_optimality}

G-RRT* uses the greedy informed set $\ginfset$, introduced in Section~\ref{sec:greedy_informed_set}, to focus the search on a subset of the state space.
This subset is bounded above by the informed subset, that is, $\ginfset \subseteq \infset$.
Focusing the search in this way increases the likelihood of sampling states that lie within the omniscient set.
However, not all states in $\ginfset$ are guaranteed to be in the omniscient set (see Figure~\ref{fig:analysis}).
As a result, some states that could yield better solutions may lie outside the homotopy regions covered by $\ginfset$ (Remark~\ref{remark:non_optimality}).
Nevertheless, because G-RRT* samples from both $\ginfset$ and $\infset$ according to the greedy biasing ratio $\epsilon$ (Section~\ref{sec:greedy_informed_sampling}), setting $\epsilon < 1$ is a sufficient condition for preserving asymptotic optimality.

\begin{theorem}[Worst-case sample complexity for probabilistic optimality]
\label{theorem:worst-case}
When sampling with a greedy biasing ratio $\epsilon$, the worst-case number of samples required to achieve probabilistic optimality increases by a factor of $\frac{1}{1-\epsilon}$ relative to sampling only from the informed set.
\end{theorem}

\begin{proof}
Let $P(\mathbf{x}_\mathrm{rand} \in \set{f})$ represent the probability of sampling a state from the omniscient set. According to Lemma 5 from \cite{gammell2018informed}, sampling states from the omniscient set is a necessary condition for improving the current solution in RRT*-like algorithms. Let $P(\mathbf{x}_\mathrm{rand} \in \infset)$ represent the probability of sampling a state from the informed set.
Since the informed set $\smash{\infset}$ is a superset of the omniscient set $\smash{\set{f}}$ \citep{gammell2018informed}, sampling from $\smash{\infset}$ is also a necessary condition for probabilistic optimality, and the probability of doing so is bounded below by the probability of sampling from $\set{f}$:
\begin{equation*}
P(\mathbf{x}_\mathrm{rand} \in \set{f}) \leq P(\mathbf{x}_\mathrm{rand} \in \infset)   
\end{equation*}
Thus, the expected number of samples required to improve the current solution to a holonomic problem depends on the probability of sampling the informed set, and is given by:
\begin{equation}
\label{eqn:thm3-expectation}
    E[n] = \frac{1}{P(\mathbf{x}_\mathrm{rand} \in \infset)}
\end{equation}
Recall that with probability $1 - \epsilon$, we sample from the informed set, and with probability $\epsilon$, we sample from the greedy informed set. Therefore, the overall probability of sampling from $\infset$ in any given iteration is reduced, since we do not always sample from $\infset$. Specifically, the new probability $P_\mathrm{new}(\mathbf{x}_\mathrm{rand} \in \infset)$ of selecting a random sample from $\infset$ is:
\begin{equation}
\label{eqn:thm3-pnew}
P_\mathrm{new}(\mathbf{x}_\mathrm{rand} \in \infset) = (1 - \epsilon) \cdot P(\mathbf{x}_\mathrm{rand} \in \infset)
\end{equation}
Substituting (\ref{eqn:thm3-pnew}) into (\ref{eqn:thm3-expectation}), the expected number of samples required to sample from $\infset$ now becomes:
\begin{equation}
\label{eqn:thm3-expectation-new}
\begin{aligned}
    E[n_\mathrm{new}] &= \frac{1}{(1 - \epsilon) \cdot P(\mathbf{x}_\mathrm{rand} \in \infset)}  
\end{aligned}
\end{equation}
Comparing (\ref{eqn:thm3-expectation}) and (\ref{eqn:thm3-expectation-new}):
\begin{equation}
\label{eqn:thm3-expectation-comparison}
\dfrac{E[n_\mathrm{new}]}{E[n]} = \dfrac{\dfrac{1}{(1 - \epsilon) \cdot P(\mathbf{x}_\mathrm{rand} \in \infset)}}{\dfrac{1}{P(\mathbf{x}_\mathrm{rand} \in \infset)}} = \dfrac{1}{1-\epsilon}  
\end{equation}
Therefore, in the worst case where the greedy informed set does not include states from the optimal solution, the number of samples required for asymptotic optimality increases by a factor of $ \frac{1}{1 - \epsilon} $ compared to only sampling from $ \infset $.
\end{proof}

\begin{theorem}[Expected sample complexity under mixed greedy sampling]
\label{theorem:average-case}
Let the planner sample from the informed set with probability $1 - \epsilon$ and from the greedy informed set with probability $\epsilon$.
Let $\gamma \in [0,1]$ denote the fraction of the omniscient set that is contained within the greedy informed set (i.e., the recall of $\smash{\ginfset}$ with respect to $\smash{\set{f}}$), and let $\rho$ denote the ratio of the volume of the greedy informed set to the volume of the informed set.
Then the expected number of samples required for probabilistic optimality satisfies
\begin{equation}
\frac{E[n_\mathrm{new}]}{E[n]} = \dfrac{1}{\,1 - \epsilon + \epsilon\,\gamma/\rho\,}.
\end{equation}
In particular, $E[n_\mathrm{new}] < E[n]$ if and only if $\gamma > \rho$.
\end{theorem}

\begin{proof}
Let $\lambda$ denote the Lebesgue measure on $\mathbb{R}^n$.
Under uniform sampling of the greedy informed set, and following Definition~9 in~\cite{gammell2018informed}, we define $\gamma$ as the \emph{recall} of $\ginfset$ with respect to the omniscient set:
\begin{equation}
\label{eqn:avg-case-recall}
\gamma = \mathrm{Recall}\left( \ginfset \right) = \frac{\lambda(\set_{f} \cap \ginfset)}{\lambda(\set_{f})} \in \lbrack 0, 1 \rbrack.
\end{equation}
Similarly, let $\rho$ denote the ratio of the Lebesgue measures of the greedy informed set and the informed set:
\begin{equation}
\rho =
\frac{\lambda(\ginfset)}
     {\lambda(\infset)} \in (0,1].
\end{equation}
We can bound the Lebesgue measure of $\smash{\infset}$ by the measure of the prolate hyperspheroid $\set{\mathrm{PHS}}$ with transverse diameter $c_i$ as
\begin{equation}
\label{eqn:informed-measure}
\lambda\bigl(\infset\bigr) \leq
\frac{\zeta_{n}}{2^{\,n}}\;
c_i\bigl(c_i^{2}-c_{\mathrm{min}}^{2}\bigr)^{\frac{n-1}{2}}.
\end{equation}
Here $\zeta_{n}$ is the Lebesgue measure of the unit $n$-ball,
\begin{equation}
\zeta_n=\frac{\pi^{n/2}}{\Gamma\!\left(\tfrac{n}{2}+1\right)}.
\end{equation}
Similarly, the Lebesgue measure of $\ginfset$ is bounded by the measure of $\set{\mathrm{PHS}}$ with transverse diameter $\Hat{f}(\Vector{x}_{\mathrm{max}})$,
\begin{equation}
\label{eqn:greedy-measure}
\begin{split}
\lambda\bigl(\ginfset\bigr) &\leq
\frac{\zeta_{n}}{2^{\,n}}\;
\Hat{f}(\Vector{x}_{\mathrm{max}})\bigl({\Hat{f}(\Vector{x}_{\mathrm{max}})}^{2}-c_{\mathrm{min}}^{2}\bigr)^{\frac{n-1}{2}}   
\end{split}
\end{equation}
Using (\ref{eqn:informed-measure}) and (\ref{eqn:greedy-measure}), $\rho$ admits the closed form
\begin{equation}
\rho
=
\frac{\Hat{f}(\Vector{x}_{\mathrm{max}})}{c_i}\;
\left(
\frac{\Hat{f}(\Vector{x}_{\mathrm{max}})^{2}-c_{\mathrm{min}}^{2}}
     {c_i^{2}-c_{\mathrm{min}}^{2}}
\right)^{\!\frac{n-1}{2}}.
\end{equation}
Let $P(\xrand \in \set{f})$ denote the probability of sampling a state from the omniscient set. According to Lemma~5 from \cite{gammell2018informed}, sampling states from the omniscient set is a necessary condition for improving the current solution in RRT*-like algorithms.
Recall that with probability $\epsilon$, we sample from the greedy informed set, and with probability $1 - \epsilon$, we sample from the informed set.
Therefore, the new probability $P_\mathrm{new}$ of selecting a random sample from either set that can guarantee improvement is:
\begin{equation}
\begin{aligned}
P_\mathrm{new} = \epsilon \cdot P (\xrand \in \set_{f} \mid \xrand \sim \mathcal{U}(\ginfset)) \\+ (1 - \epsilon) \cdot P (\xrand \in \set_{f} \mid \xrand \sim \mathcal{U}(\infset))
\end{aligned}
\end{equation}
Under uniform sampling, each conditional probability can be expressed as a ratio of relative measures. Specifically,
\begin{align}
\label{eqn:avg-case-prob-ratios}
P_\mathrm{new} = \epsilon \cdot \frac{\lambda(\set_{f} \cap \ginfset)}{\lambda(\ginfset)} + (1 - \epsilon) \cdot \frac{\lambda(\set_{f} \cap \infset)}{\lambda(\infset)}
\end{align}
Since $\infset \supseteq \set{f}$, as noted in \cite{gammell2018informed}, (\ref{eqn:avg-case-prob-ratios}) simplifies to:
\begin{align}
\label{eqn:avg-case-prob-ratios-simplified}
P_\mathrm{new} = \epsilon \cdot \frac{\lambda(\set_{f} \cap \ginfset)}{\lambda(\ginfset)} + (1 - \epsilon) \cdot \frac{\lambda(\set_{f})}{\lambda(\infset)}
\end{align}
Expressing (\ref{eqn:avg-case-prob-ratios-simplified}) in terms of $\gamma$ and $\rho$ gives:
\begin{equation}
\begin{aligned}
P_\mathrm{new}
&= \epsilon \cdot \left( \gamma / \rho \frac{\lambda(\set_{f})}{\lambda(\infset)} \right) + (1 - \epsilon) \cdot \frac{\lambda(\set_{f})}{\lambda(\infset)}\\
&= \left\lbrack (1 - \epsilon) + \epsilon (\gamma/\rho) \right\rbrack \cdot \frac{\lambda(\set_{f})}{\lambda(\infset)}
\end{aligned}
\end{equation}
Thus, the expected number of samples required to improve the current solution to a holonomic problem depends on the probability of sampling the informed set as well as on overlap measures between the two sets $\set_{f}$ and $\ginfset$:
\begin{equation}
\label{eqn:avg-case-result}
\frac{E[n_\mathrm{new}]}{E[n]} = \frac{1}{\left\lbrack (1 - \epsilon) + \epsilon (\gamma/\rho) \right\rbrack}.
\end{equation}
\end{proof}
In particular, setting $\gamma = 0$ in (\ref{eqn:avg-case-result}) recovers the worst-case factor $1/(1 - \epsilon)$, while setting $\gamma = 1$ and $\epsilon = 1$ (sampling only from $\ginfset$) yields the best-case factor $\rho$.

%% file: section7/experiments.tex
\begin{figure*}[!tb]
    \centering
    \begin{subfigure}[b]{0.27\textwidth}
        \centering
        \includegraphics[width=\textwidth]{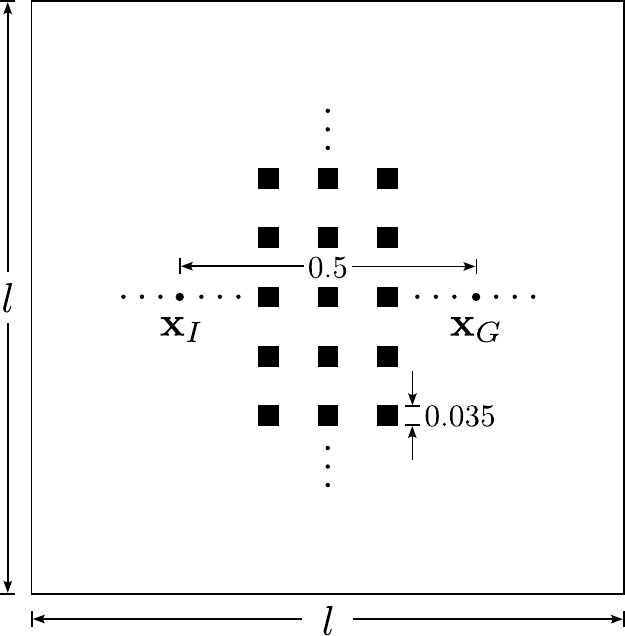}
        \captionsetup{justification=centering}
        \caption{}
        \label{subfig:abstract_experiments_repeating_rectangles}
    \end{subfigure}\hspace{0.02\textwidth}
    \begin{subfigure}[b]{0.27\textwidth}
        \centering
        \includegraphics[width=\textwidth]{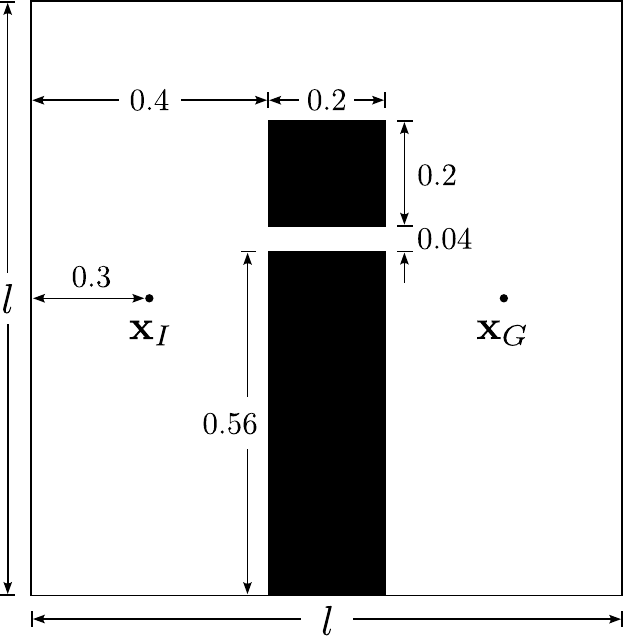}
        \captionsetup{justification=centering}
        \caption{}
        \label{subfig:abstract_experiments_narrow_passage_gap}
    \end{subfigure}\hspace{0.02\textwidth}
    \begin{subfigure}[b]{0.27\textwidth}
        \centering
        \includegraphics[width=\textwidth]{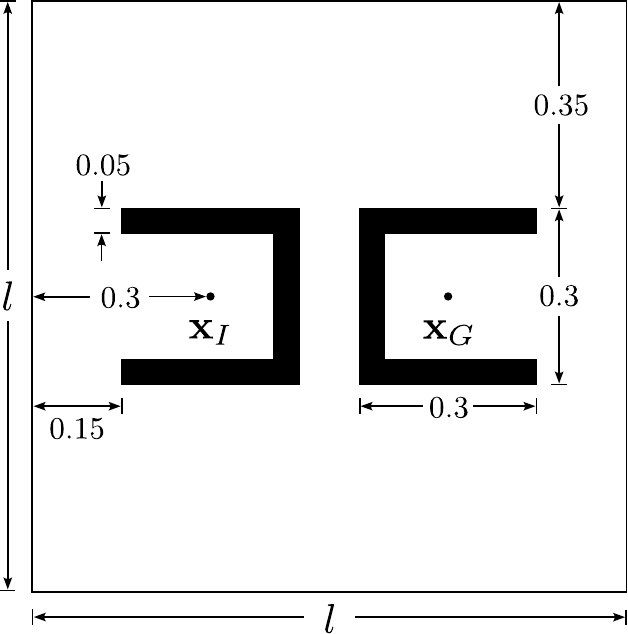}
        \captionsetup{justification=centering}
        \caption{}
        \label{subfig:abstract_experiments_double_enclosure}
    \end{subfigure}
    \caption{Two-dimensional illustrations of the abstract planning problems examined in Section~\ref{sec:abstract-problems}. These include complex planning problems containing (a) many homotopy classes, (b) narrow passage gap environment, and (c) double enclosures, with a problem domain of size $l = 1$. The $\mathbf{x}_I$ and $\mathbf{x}_G$ represent the initial and goal states, respectively.}
    \label{fig:abstract_experiments}
\end{figure*}

\begin{figure*}[!tb]
\centering
\begin{subfigure}[b]{\textwidth}
    \centering
    \includegraphics[width=0.85\textwidth]{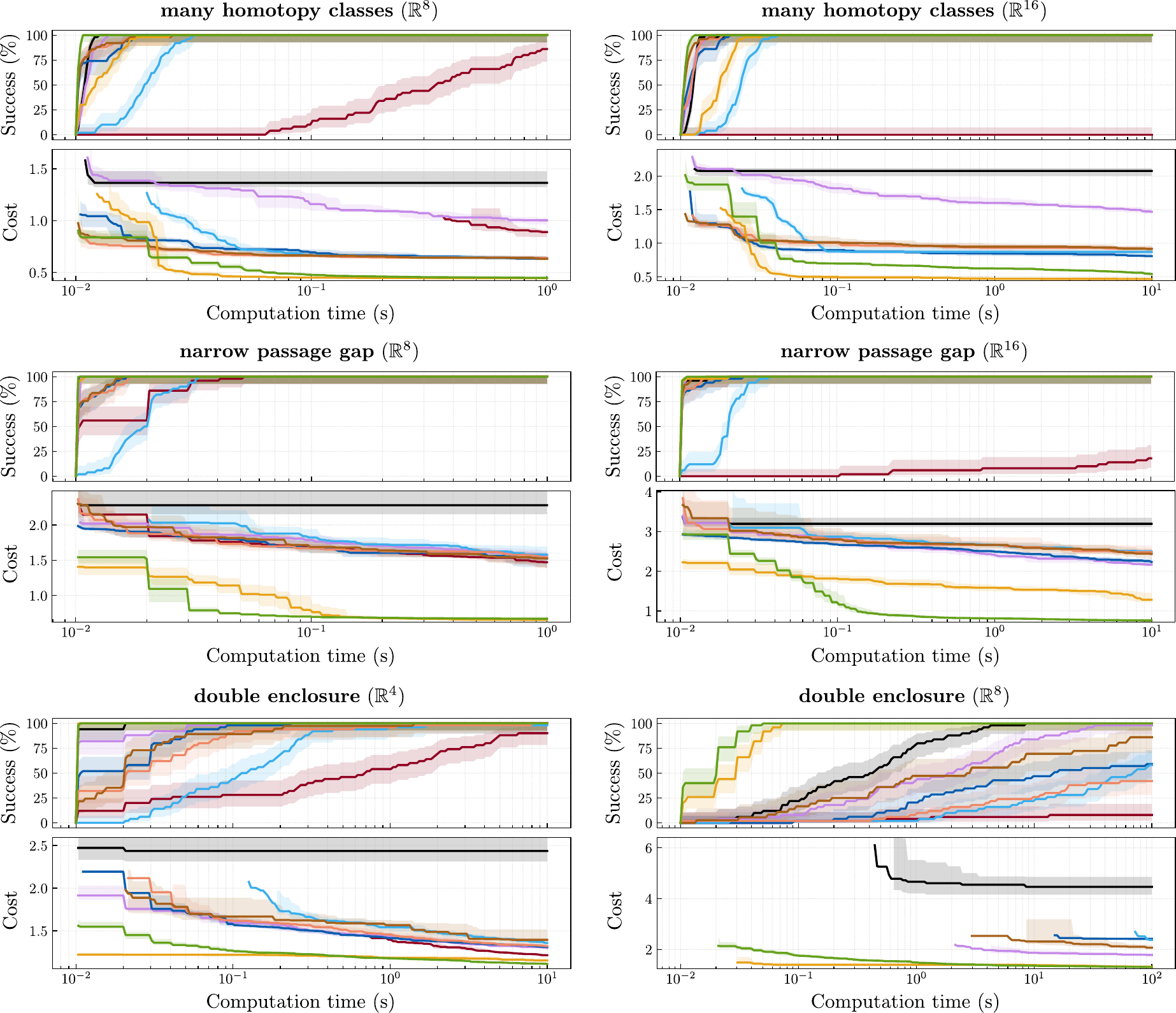}
    \vspace{2mm}
\end{subfigure}
\begin{subfigure}[b]{0.65\textwidth}
    \hspace*{1em}  
    \centering
    \includegraphics[width=\textwidth]{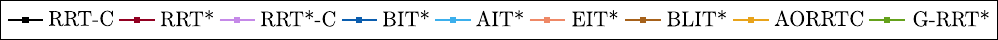}
\end{subfigure}
\vspace{2mm}
\caption{Planner performance versus runtime on the abstract planning problems described in Section~\ref{sec:abstract-problems}.
The success plots show the percentage of successful runs over time, while the cost plots present the median solution cost versus runtime for each planner.
Error bars and shaded regions denote non-parametric 99\% confidence intervals on the median.
Reported times and costs for AORRTC include path simplification, i.e., randomized shortcutting and B-spline smoothing.}
\label{fig:benchmark_graphs_abstract_problems}
\end{figure*}

\input{section7/table_abstract_v2}

\section{Experiments}
\label{sec:experiments}

We evaluated the performance of G-RRT* by comparing it with several state-of-the-art sampling-based planners on abstract problems in $\mathbb{R}^4$, $\mathbb{R}^{8}$, and $\mathbb{R}^{16}$ (Section~\ref{sec:abstract-problems}), as well as on robotic manipulation problems in $\mathbb{R}^7$, $\mathbb{R}^{8}$, and $\mathbb{R}^{14}$ (Section~\ref{sec:manipulation_exp}).
The former problems test the planner's ability to find high-quality solutions in challenging, high-dimensional spaces, whereas the latter assesses performance on realistic robotic tasks.
We compared G-RRT* with the Open Motion Planning Library~\citep[OMPL]{sucan2012open} implementations of RRT-Connect, RRT*, RRT*-Connect \citep{klemm2015rrt}, BIT* \citep{gammell2020batch}, AIT* and EIT* \citep{strub2022adaptively}, BLIT* \citep{wang2025asymptotically}, and AORRTC \citep{wilson2025aorrtc}.
All experiments were conducted using the Planner Developer Tools~\citep[PDT]{gammell2022planner} and Robowflex~\citep{kingston2022robowflex}\footnote{Simulations were performed in an Ubuntu 20.04 Docker container on an Intel Core i7-10875H CPU with 40 GB of RAM; all planners were implemented in C$++$.}.

All planners were evaluated with the optimization objective of minimizing path length in $\mathbb{R}^n$.
Each planner used the default edge length specified by OMPL, computed as a fixed fraction of the maximum possible distance between any two states in the space.
An RGG constant of $\eta = 1.001$ was applied to all planners, and the Euclidean distance (i.e., $L^2$ norm) for path length was used as the admissible cost heuristic.
BIT*-based planners used a batch size of 100, and G-RRT* employed a greedy bias ratio of $\epsilon = 0.9$ to encourage exploitation across all experiments.
Notably, AORRTC internally applies path simplification---specifically, randomized shortcutting~\citep{geraerts2007creating,hauser2010fast} and B-spline smoothing~\citep{pan2012collision}---as part of its planning process to improve solution quality.
To ensure a fair comparison, and because all reported times and costs for AORRTC include this simplification, we applied the same procedure to all planners as a post-processing step \emph{only}: (i) after an initial solution had been found, and (ii) for the final solution at the end of the planning time limit.
The costs of these post-processed, simplified paths are reported separately.\footnote{Simplification was evaluated only in separate post-processing experiments and did not affect convergence; the simplified paths were not reused during planning.}

\begin{figure*}[!tb]
   \centering
   \includegraphics[width=0.765\textwidth]{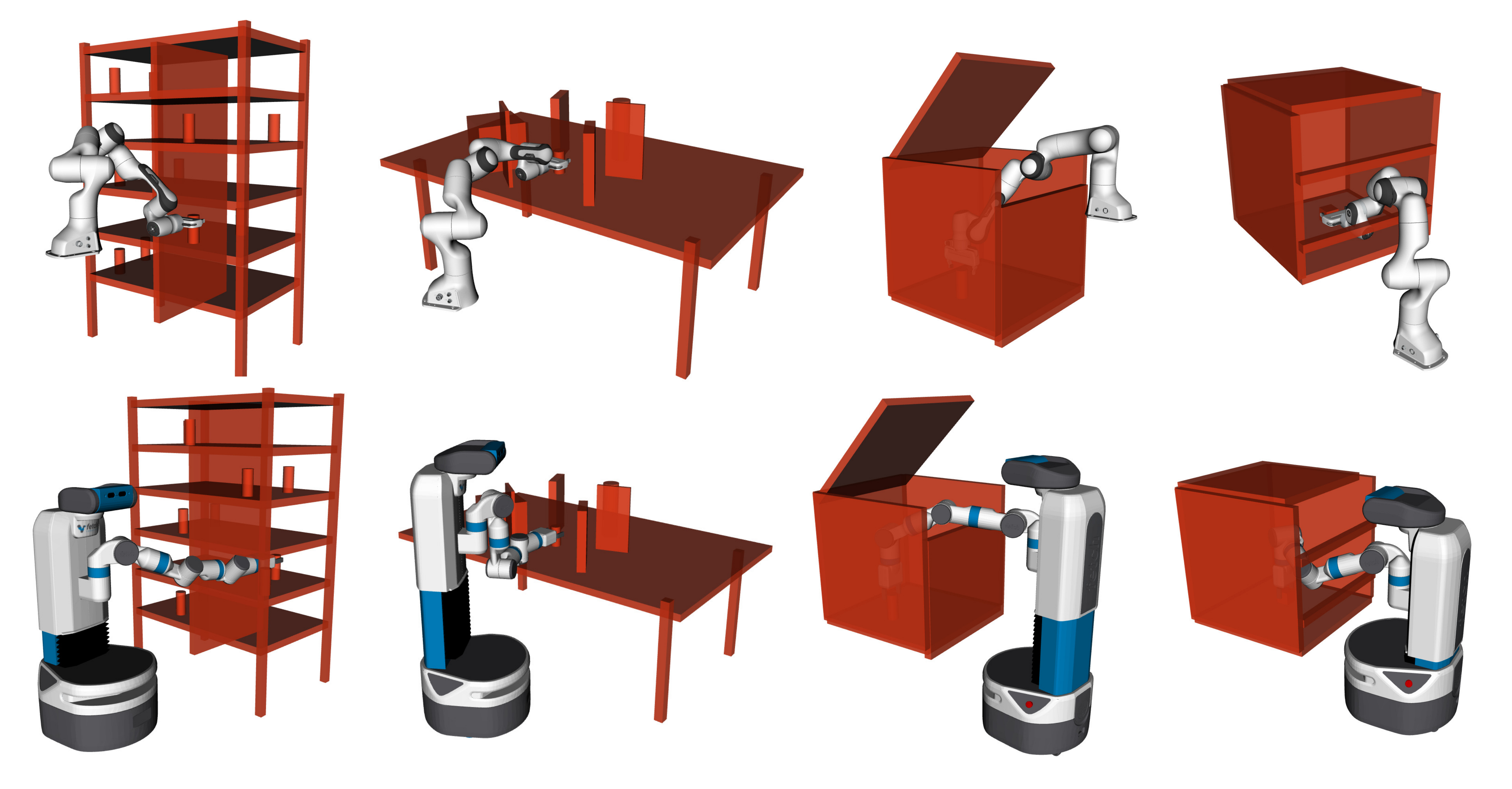}
\caption{Manipulation problems from the MotionBenchMaker dataset used in our experiments (see Section~\ref{sec:manipulation_exp}).
Each row shows a different robot platform (Panda or Fetch) evaluated on four benchmark tasks: \emph{bookshelf thin}, \emph{table pick}, \emph{box}, and \emph{cage}.}
   \label{fig:mbm_problems}
   \vspace*{-4.8mm}
\end{figure*}

\begin{figure}[!tb]
    \vspace{2mm}
    \centering
    \begin{subfigure}[b]{0.22\textwidth}
        \centering
        \fbox{\includegraphics[width=0.94\textwidth]{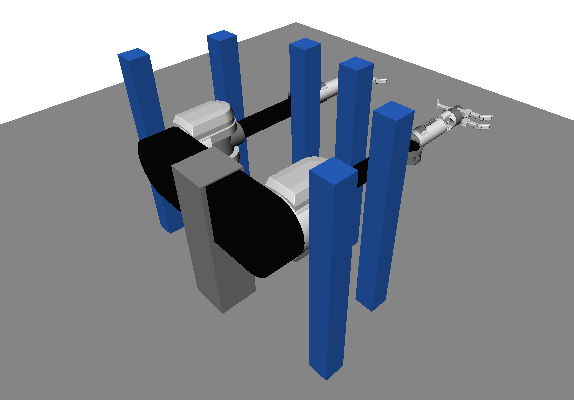}}
        \captionsetup{justification=centering}
        \caption{Start configuration}
        \label{subfig:cage_start_orbit}
    \end{subfigure}\hspace{0.01\textwidth}
    \begin{subfigure}[b]{0.22\textwidth}
        \centering
        \fbox{\includegraphics[width=0.94\textwidth]{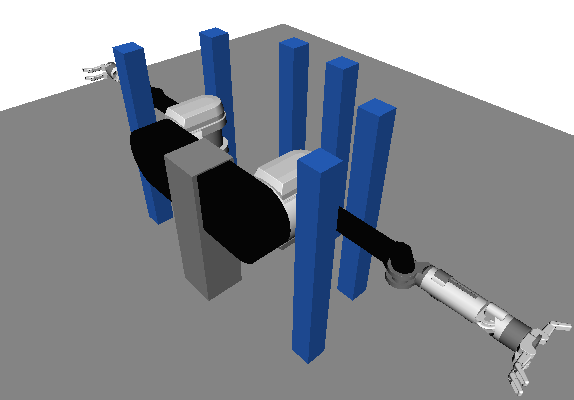}}
        \captionsetup{justification=centering}
        \caption{Goal configuration}
        \label{subfig:cage_goal_orbit}
    \end{subfigure}
    \caption{A dual-arm manipulation problem for the Barrett WAM Arm in $\mathbb{R}^{14}$. Starting with both arms pointing in the same direction (a), they must be moved to extend in opposite directions (b) without hitting the cage.}
    \vspace{-1.0em}
    \label{fig:manipulation_experiments_cage}
\end{figure}

\begin{figure*}[!tb]
    \centering
    \begin{subfigure}[b]{\textwidth}
        \centering
        \includegraphics[width=0.85\textwidth]{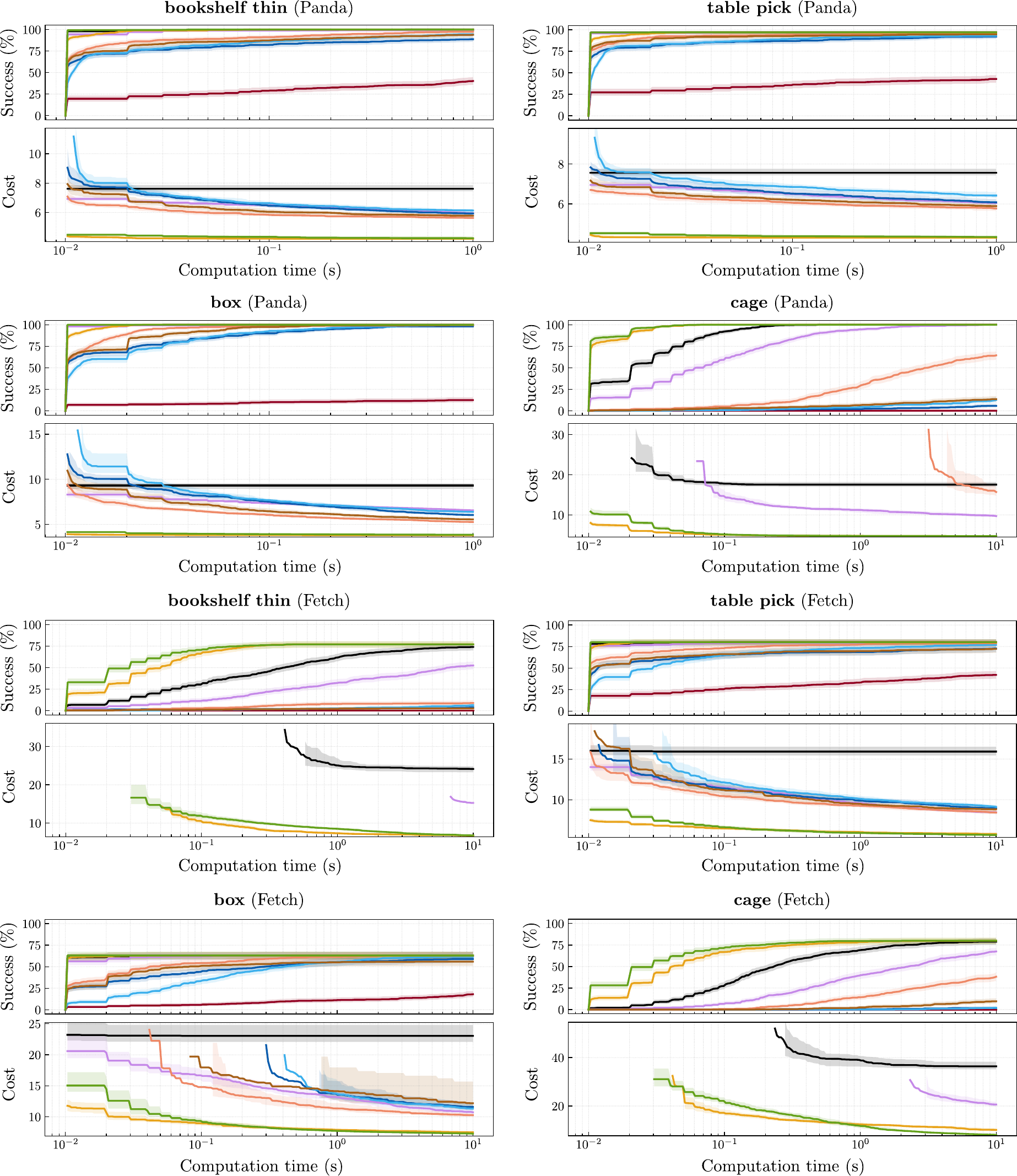}
        \vspace{2mm}
    \end{subfigure}
    \begin{subfigure}[b]{0.65\textwidth}
        \centering
        \includegraphics[width=\textwidth]{figures/experiments/legend.pdf}
    \end{subfigure}
    \vspace{2mm}
\caption{Planner performance versus running time on the Panda ($\mathbb{R}^{7}$) and Fetch ($\mathbb{R}^{8}$) problems from the MotionBenchMaker dataset (see Section~\ref{sec:manipulation_exp} and Figure~\ref{fig:mbm_problems}). 
The success plots show the percentage of successful runs over time, while the cost plots present the median solution cost versus runtime for each planner. 
Error bars and shaded regions denote non-parametric 99\% confidence intervals on the median. 
Reported times and costs for AORRTC include path simplification, i.e., randomized shortcutting and B-spline smoothing.}
\vspace{-1.0em}
\label{fig:manipulation-results-mbm-problems}
\end{figure*}

\begin{figure}[!tb]
    \centering
    \begin{subfigure}[b]{\linewidth}
        \centering
        \includegraphics[width=0.95\textwidth]{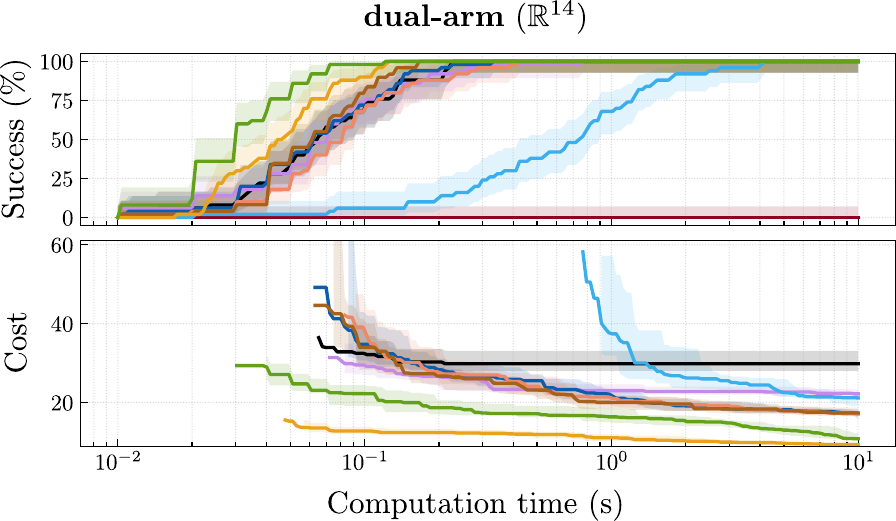}
        \vspace{-1em}
        \label{subfig:exp1R8}
    \end{subfigure}
    \begin{subfigure}[b]{\linewidth}
        \centering
        \includegraphics[width=\textwidth]{figures/experiments/legend.pdf}
    \end{subfigure}
    \vspace{-2mm}
\caption{Planner performance versus runtime on the dual-arm problem described in Section~\ref{sec:manipulation_exp} (see Figure~\ref{fig:manipulation_experiments_cage}).
The success plots show the percentage of successful runs over time, and the cost plots present the median solution cost versus runtime for each planner.
Error bars and shaded regions denote non-parametric 99\% confidence intervals on the median.
Each planner was run for 50 trials with a runtime limit of 10~s in $\mathbb{R}^{14}$.
Reported times and costs for AORRTC include path simplification, i.e., randomized shortcutting and B-spline smoothing.}
\vspace{-1.0em}
\label{fig:manipulation-results-dual-arm}
\end{figure}

\subsection{Abstract Planning Problems}
\label{sec:abstract-problems}

The planners were tested on three simulated planning problems in $\mathbb{R}^4$, $\mathbb{R}^8$, and $\mathbb{R}^{16}$.
To provide some intuition, these problems are illustrated in two dimensions in Figure~\ref{fig:abstract_experiments}.
The $\mathbb{R}^2$ versions are shown only for visualization purposes, since they are too simple to differentiate planner performance in a meaningful way.
The higher-dimensional problem instances were created by extending the same obstacle layouts uniformly along each additional dimension.
The first problem consists of many axis-aligned hypercubes, with the start and goal located at $[-0.25, 0, \dots, 0]^{\top}$ and $[0.25, 0, \dots, 0]^{\top}$, respectively (Figure~\ref{subfig:abstract_experiments_repeating_rectangles}).
The second problem contains a wall with a narrow passage gap, such that only two homotopy classes exist in all dimensions (Figure~\ref{subfig:abstract_experiments_narrow_passage_gap}), with start and goal at $[-0.3, 0, \dots, 0]^{\top}$ and $[0.3, 0, \dots, 0]^{\top}$.
The final problem includes two hollow, axis-aligned hypercubes with an opening, enclosing the start and goal, located at $[-0.3, 0, \dots, 0]^{\top}$ and $[0.3, 0, \dots, 0]^{\top}$, respectively (Figure~\ref{subfig:abstract_experiments_double_enclosure}).

Each planner was allocated a planning time proportional to the difficulty of the problem and the dimensionality of the state space.
We ran 100 trials for the first two problems and 50 for the last, each initialized with a different pseudorandom seed.
During each trial, solution costs were recorded at fixed time intervals ($10^{-4}$ s) by a separate monitoring thread.
If no solution was found at a given time, the cost was assigned an infinite value.
We then reported the median costs to mitigate any bias from unsolved trials.
Figure~\ref{fig:benchmark_graphs_abstract_problems} shows each planner's median solution cost and success rate over computation time.
To ensure a fair comparison with AORRTC, we also applied path simplification to all of the initial and final solutions, as summarized in Table~\ref{tab:abstract-problems}.

\looseness=-1
The results highlight the importance of maintaining two rapidly growing trees to quickly find initial solutions in high-dimensional problems.
As the median times to the initial solution in Table~\ref{tab:abstract-problems} show, G-RRT* discovers initial solutions at least as fast as RRT-Connect across all tested problems, and faster than every other planner on the harder ones.
On most problems, G-RRT* is even faster, although the difference comes from the balanced search strategy of~\citet{kuffner2005efficient}, which grows whichever of the two trees is smaller, whereas the default OMPL implementation of RRT-Connect alternates between the start and goal trees.
Balanced search helps most when obstacles slow the growth of one of the trees, since alternating then spends half of the iterations on a tree that grows very little, and we would expect RRT-Connect to reach initial solutions just as quickly if it grew the trees the same way.
AORRTC uses the same strategy, which is why its initial solution times are close to those of G-RRT*, with a gap of a few milliseconds because it also applies path simplification to the solutions it reports.
Unlike RRT-Connect, G-RRT* is an anytime algorithm that converges asymptotically to optimal solutions.
From Figure~\ref{fig:benchmark_graphs_abstract_problems}, we observe that AORRTC reports lower median solution costs than G-RRT* because its paths are already simplified.
However, G-RRT* reaches the same final solution costs as AORRTC, and does so faster than AORRTC and all other asymptotically optimal planners, even \emph{without} path simplification.
These results demonstrate the effectiveness of greedy informed sampling.
Moreover, as shown in Table~\ref{tab:abstract-problems}, when the same simplification process is applied uniformly to all planners, their initial solution costs become similar; yet G-RRT* consistently yields lower final solution costs for most problems.
Simplification improves path quality only locally, whereas G-RRT* achieves global improvement through RRT*-style rewiring combined with its greedy biasing strategy.

Some of the most constrained problems further illustrate the strengths of G-RRT*.
In the narrow passage gap problem, all asymptotically optimal planners struggle to find a path through the passage, yet G-RRT* is the only one that consistently succeeds within the available time budget.
This demonstrates the effectiveness of the greedy informed set relative to the standard informed set used by AORRTC and the other asymptotically optimal planners.

In the more difficult \emph{bug trap} scenario involving double enclosures symmetric about the start and goal, all planners struggled as dimensionality increased.
G-RRT*, RRT-Connect, and AORRTC found initial solutions much faster than the other batch-based planners, emphasizing the importance of maintaining two greedily connecting trees in domains with several biased Voronoi regions~\citep{yershova2005dynamic}.
Nevertheless, after simplification, G-RRT* achieves lower median costs than competing planners, confirming its ability to find higher-quality solutions within the allotted planning time.

\input{section7/table_manipulation_v2}

\subsection{Manipulation Problems}
\label{sec:manipulation_exp}

The planners were also evaluated on manipulation tasks from the MotionBenchMaker dataset~\citep{chamzas2021motionbenchmaker}, using the Panda ($\mathbb{R}^{7}$) and Fetch ($\mathbb{R}^{8}$) robot platforms.
The dataset provides a standardized set of seven manipulation benchmarks across multiple robot models, each containing 100 planning problems with varying start and goal configurations and workspace obstacles.
For each robot, we selected four environments of increasing difficulty---\textit{bookshelf thin}, \textit{table pick}, \textit{box}, and \textit{cage}---ranging from relatively open reaching motions to tightly constrained scenarios (Figure~\ref{fig:mbm_problems}).
In addition, the planners were evaluated on a more challenging, higher-dimensional dual-arm manipulation task involving two Barrett WAM arms with a combined 14 degrees of freedom ($\mathbb{R}^{14}$) in the Open Robotics Automation Virtual Environment~\citep{diankov2010automated}, as shown in Figure~\ref{fig:manipulation_experiments_cage}.
The objective is to move the arms from an initial configuration, where they point in the forward direction, to a goal configuration where they are extended outward in opposite directions, without colliding with the surrounding cage.
For all manipulation problems, including the dual-arm task, we used VAMP~\citep{thomason2024motions} as the collision-checking backend for the OMPL planners.
Performance results for all planners, optimizing path length in $\mathbb{R}^n$, are presented in Figures~\ref{fig:manipulation-results-mbm-problems} and~\ref{fig:manipulation-results-dual-arm}, with post-simplification statistics summarized in Table~\ref{tab:manipulation-problems}.

The observed performance trends closely match those of the abstract problems discussed in Section~\ref{sec:abstract-problems}.
G-RRT* finds initial solutions at least as fast as RRT-Connect, and faster than the other planners on the harder problems, then uses the remaining planning time to converge toward optimality.
In most planning problems from the MotionBenchMaker dataset, G-RRT* achieves final solution costs comparable to those of AORRTC, but does so \emph{without} requiring any path simplification---indicating faster convergence relative to the other asymptotically optimal planners.
Even in the hardest Fetch problems where most planners fail to find any feasible path, G-RRT* yields lower final costs than every other planner, including AORRTC.
Notably, for the dual-arm problem, G-RRT* not only finds solutions faster but also produces the lowest final costs after simplification, highlighting the benefit of the greedy heuristic for high-dimensional planning problems.

\subsection{Ablation Studies}
\label{sec:ablation-studies}

Finally, we present two ablation studies of greedy informed sampling.
We first examine the sensitivity of G-RRT* to the greedy biasing ratio $\epsilon$ in Section~\ref{sec:ablation-first}, then isolate the benefit of the greedy informed set by comparing it against the standard informed set with path simplification in Section~\ref{sec:ablation-second}.

\subsubsection{Effect of Greedy Biasing Ratio}
\label{sec:ablation-first}

The greedy biasing ratio $\epsilon$ decides how often G-RRT* draws samples from the greedy informed set rather than the informed set.
To assess sensitivity to this parameter, we varied $\epsilon$ across eleven values from $0$ to $1$ in steps of $0.1$ and reran every abstract and manipulation problem, with all other settings fixed.
Figure~\ref{fig:sensitivity-grid} shows the resulting median solution cost over computation time for each problem, with curves colored by $\epsilon$.
On the easier problems, such as the Panda tasks, the curves overlap and convergence is largely insensitive to $\epsilon$.
On the harder problems, a larger $\epsilon$ converges faster to lower-cost solutions, and the effect grows with the difficulty of the problem, most clearly on the dual-arm cage, narrow passage gap, many homotopy classes, and Fetch cage problems.
Disabling the greedy bias, $\epsilon=0$, gives the slowest convergence and the highest final cost throughout.

\begin{figure*}[!tb]
\centering
\includegraphics[width=\textwidth]{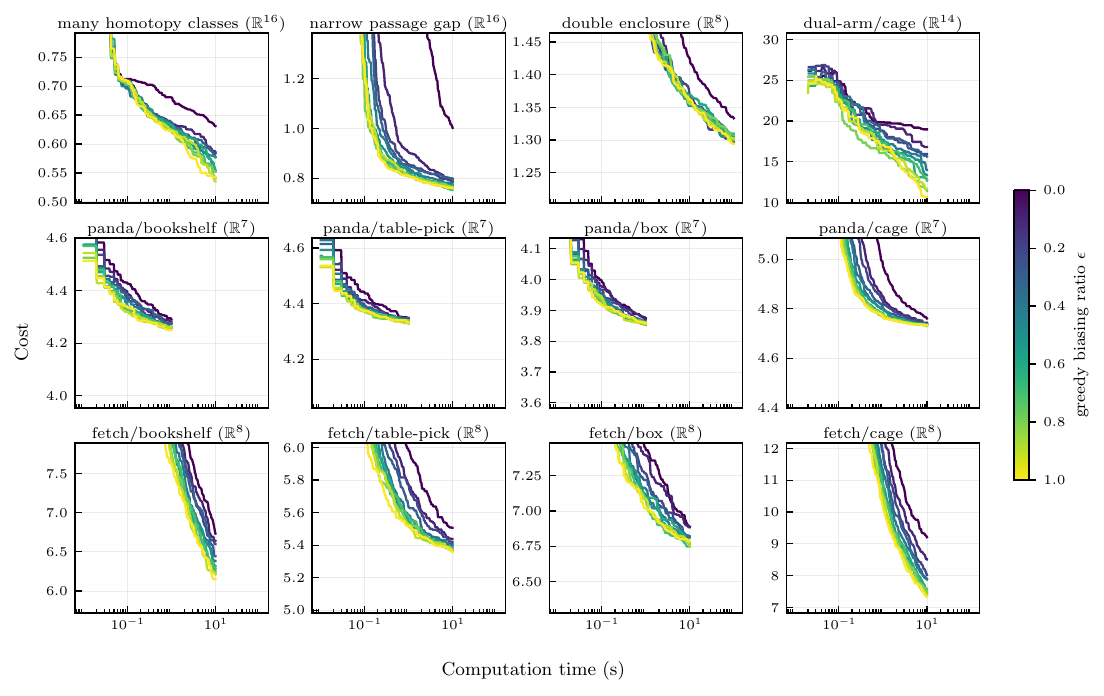}
\caption{Sensitivity of G-RRT* to the greedy biasing ratio $\epsilon$. Each panel shows the median solution cost over computation time for one benchmark, with curves colored by $\epsilon$ from $0$ to $1$ (see colorbar). The easier problems, such as the Panda tasks, are largely insensitive to $\epsilon$, while on the harder problems a larger $\epsilon$ converges faster to lower-cost solutions, most strongly on the dual-arm cage problem.
}
\label{fig:sensitivity-grid}
\end{figure*}

\subsubsection{\texorpdfstring{Path Simplification vs.\\Greedy Informed Sampling}{Path Simplification vs. Greedy Informed Sampling}}
\label{sec:ablation-second}

\looseness=-1
A natural alternative to greedy informed sampling is to shorten the current solution with a path simplification step and then sample from the standard informed set built around the simplified cost.
We evaluate this as a controlled variant of G-RRT*, denoted Simpl.+Inf., that replaces the greedy informed set with the standard informed set and adds the same randomized shortcutting and B-spline smoothing used for the other planners.
Both variants share the same bidirectional search, heuristic gating, and delayed rewiring, so the comparison isolates the effect of the greedy informed set.
We report the cost of {G-RRT*} without any post-processing.
Table~\ref{tab:informed-set-ablation} reports the median final cost of both variants on the manipulation problems.
G-RRT* yields a lower final cost on every problem, even though its solutions are never simplified.
The two variants are within half a percent on the easier Panda problems, where a short initial path leaves little to improve, but the gap grows with the difficulty of the problem, reaching $12\%$ on the Fetch cage problem and $36\%$ on the dual-arm cage problem.

This gap reflects a structural difference between the two strategies.
Path simplification shortens a solution only locally and within its homotopy class, whereas the greedy informed set concentrates global search on the region that can lower the cost, which G-RRT* then exploits through continued rewiring.
Figure~\ref{fig:sampling-bound} shows this on the Fetch cage problem, where the greedy informed set maintains a tighter cost bound and a lower solution cost throughout planning than the Simpl.+Inf. variant.
Path simplification with informed sampling therefore does not match greedy informed sampling, and its disadvantage grows with the difficulty of the problem.

\begin{figure}[!tb]
    \centering
    \begin{subfigure}[b]{\linewidth}
        \centering
        \vspace{-1.5em}
        \includegraphics[width=\textwidth]{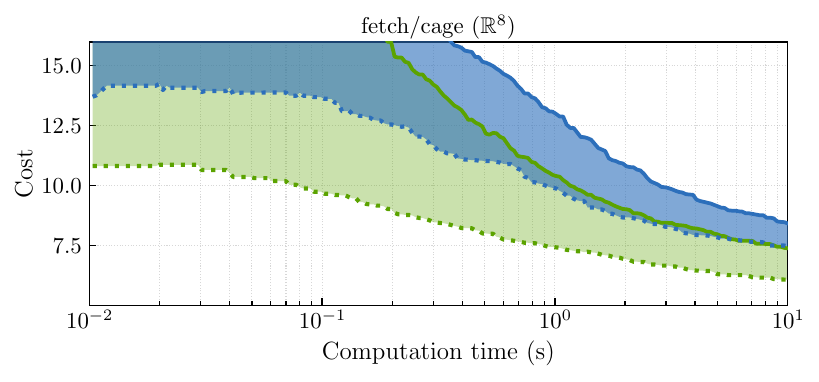}
        \vspace{-1.5em}
    \end{subfigure}
    \begin{subfigure}[b]{\linewidth}
        \centering
        \includegraphics[width=0.85\textwidth]{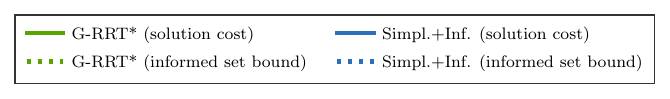}
    \end{subfigure}
    \vspace{-4mm}
\caption{
\looseness=-1
G-RRT* against the Simpl.+Inf. variant (standard informed set with path simplification) on the Fetch cage problem. Solid lines show the median solution cost and dotted lines the corresponding informed-set cost bound. The greedy informed set maintains a tighter bound and a lower solution cost throughout planning.}
\vspace{-1.0em}
\label{fig:sampling-bound}
\end{figure}

\setlength{\textfloatsep}{2mm}%
\begin{table}[!t]
\centering
\setlength{\tabcolsep}{1.5pt}
\renewcommand{\arraystretch}{1.1}
\vspace{-1.0em}
\caption{Final median solution cost of G-RRT* against the Simpl.+Inf. variant on the manipulation problems. $\Delta$ is the reduction in final cost of G-RRT* relative to Simpl.+Inf. Lower cost is better, and the lower cost in each row is shown in \textbf{bold}.}
\label{tab:informed-set-ablation}
\begin{tabular}{@{}l@{\hspace{7pt}} c@{\hspace{7pt}} c@{\hspace{7pt}} c@{}}
\toprule
& G-RRT* & Simpl.+Inf. & $\Delta$ \\
\midrule
\multicolumn{4}{@{}l}{\textbf{Panda} ($\mathbb{R}^7$)}\\
\quad bookshelf thin        & \bfseries4.258 & 4.260 & 0.0\% \\
\quad table pick            & \bfseries4.341 & 4.345 & 0.1\% \\
\quad box                   & \bfseries3.857 & 3.863 & 0.2\% \\
\quad cage                  & \bfseries4.729 & 4.754 & 0.5\% \\
\midrule
\multicolumn{4}{@{}l}{\textbf{Fetch} ($\mathbb{R}^8$)}\\
\quad bookshelf thin        & \bfseries6.826 & 7.108 & 4.0\% \\
\quad table pick            & \bfseries5.637 & 5.781 & 2.5\% \\
\quad box                   & \bfseries7.252 & 7.457 & 2.7\% \\
\quad cage                  & \bfseries8.234 & 9.360 & 12.0\% \\
\midrule
\multicolumn{4}{@{}l}{\textbf{Dual-arm} ($\mathbb{R}^{14}$)}\\
\quad cage                  & \bfseries10.932 & 17.165 & 36.3\% \\
\bottomrule
\end{tabular}
\vspace{-1.0em}
\end{table}

%% file: section7/table_abstract_v2.tex
\begin{table*}[!t]
    \sisetup{table-alignment-mode=format,table-align-text-post=false,table-number-alignment=center,table-format=2.3}
    \renewrobustcmd{\bfseries}{\fontseries{b}\selectfont}
    \renewrobustcmd{\boldmath}{}
    \centering
    \caption{Initial and final median solution costs for each planner on the abstract planning problems described in Section~\ref{sec:abstract-problems}, along with the median time to find the initial solution ($t_{\text{init}}$, in seconds). Costs are shown after randomized shortcutting and B-spline smoothing. Lower values indicate better performance, with the best cost highlighted in {\setlength{\fboxsep}{1.5pt}\colorbox{green!45}{dark green}} and the second-best in {\setlength{\fboxsep}{1.5pt}\colorbox{green!18}{light green}}.}
    \vspace{1mm}
    \label{tab:abstract-problems}
    \begin{threeparttable}
    \renewcommand{\arraystretch}{1.3}
    \setlength{\tabcolsep}{4pt}
    \resizebox{\textwidth}{!}{%
    \begin{NiceTabular}
    {
      @{} l
      S @{\hspace{0.4em}} S @{\hspace{0.4em}} S @{\hspace{0.4em}}
      S @{\hspace{0.4em}} S @{\hspace{0.4em}} S @{\hspace{0.4em}}
      S @{\hspace{0.4em}} S @{\hspace{0.4em}} S @{\hspace{0.4em}}
      S @{\hspace{0.4em}} S @{\hspace{0.4em}} S @{\hspace{0.4em}}
      S @{\hspace{0.4em}} S @{\hspace{0.4em}} S @{\hspace{0.4em}}
      S @{\hspace{0.4em}} S @{\hspace{0.4em}} S @{\hspace{0.4em}}
    }
    \CodeBefore
    \rowcolors{4}{}{gray!10,}
    \cellcolor{green!45}{9-3,12-4,10-6,12-7,12-9,12-10,7-12,12-13,5-15,12-16,6-18,12-19}
    \cellcolor{green!18}{10-3,11-4,9-6,11-7,10-9,11-10,12-12,11-13,12-15,9-16,12-18,6-19}
    \Body
    \toprule
    & \multicolumn{6}{c}{\textbf{many homotopy classes}}
    & \multicolumn{6}{c}{\textbf{narrow passage gap}}
    & \multicolumn{6}{c}{\textbf{double enclosure}}\\
    & \multicolumn{3}{c}{$\mathbb{R}^{8}$} & \multicolumn{3}{c}{$\mathbb{R}^{16}$} & \multicolumn{3}{c}{$\mathbb{R}^{8}$} & \multicolumn{3}{c}{$\mathbb{R}^{16}$} & \multicolumn{3}{c}{$\mathbb{R}^{4}$} & \multicolumn{3}{c}{$\mathbb{R}^{8}$}\\
    \cmidrule(lr){2-4}\cmidrule(lr){5-7}\cmidrule(lr){8-10}\cmidrule(lr){11-13}\cmidrule(lr){14-16}\cmidrule(lr){17-19}
    & {$t_{\text{init}}$} & {$c_{\text{init}}$} & {$c_{\text{final}}$} &
      {$t_{\text{init}}$} & {$c_{\text{init}}$} & {$c_{\text{final}}$} &
      {$t_{\text{init}}$} & {$c_{\text{init}}$} & {$c_{\text{final}}$} &
      {$t_{\text{init}}$} & {$c_{\text{init}}$} & {$c_{\text{final}}$} &
      {$t_{\text{init}}$} & {$c_{\text{init}}$} & {$c_{\text{final}}$} &
      {$t_{\text{init}}$} & {$c_{\text{init}}$} & {$c_{\text{final}}$}\\
    \midrule
    RRT-C   & 0.011 & 0.792 & 0.792 & 0.012 & 1.059 & 1.059 & 0.010 & 1.540 & 1.540 & 0.010 & 2.061 & 2.061 & 0.010 & 1.227 & 1.227 & 0.422 & 1.429 & 1.429 \\
    RRT*    & 0.339 & 0.589 & 0.589 & $\infty$ & $\infty$ & $\infty$ & 0.010 & 1.391 & 1.106 & $\infty$ & $\infty$ & $\infty$ & 0.695 & \bfseries1.158 & 1.124 & $\infty$ & $\infty$ & $\infty$ \\
    RRT*-C  & 0.011 & 0.813 & 0.629 & 0.012 & 1.089 & 0.850 & 0.010 & 1.409 & 1.268 & 0.010 & 2.123 & 2.040 & 0.010 & 1.199 & 1.125 & 1.957 & \bfseries1.346 & 1.299 \\
    BIT*    & 0.010 & 0.569 & 0.467 & 0.011 & 0.780 & 0.539 & 0.010 & 1.341 & 1.091 & 0.010 & \bfseries1.817 & 1.525 & 0.011 & 1.175 & 1.122 & 15.251 & 1.423 & 1.426 \\
    AIT*    & 0.020 & 0.643 & 0.462 & 0.024 & 0.921 & 0.557 & 0.020 & 1.373 & 1.170 & 0.020 & 1.984 & 1.749 & 0.116 & 1.175 & 1.129 & 69.468 & 1.484 & 1.433 \\
    EIT*    & 0.010 & \bfseries0.521 & 0.460 & 0.012 & 0.751 & 0.548 & 0.010 & 1.301 & 1.061 & 0.010 & 1.957 & 1.682 & 0.021 & 1.187 & 1.115 & $\infty$ & $\infty$ & $\infty$ \\
    BLIT*   & 0.010 & 0.528 & 0.462 & 0.011 & \bfseries0.745 & 0.568 & 0.010 & 1.288 & 1.073 & 0.010 & 1.921 & 1.564 & 0.022 & 1.197 & 1.165 & 14.642 & 1.407 & 1.397 \\
    AORRTC  & 0.012 & 0.986 & 0.444 & 0.018 & 1.302 & 0.468 & 0.010 & 1.397 & 0.654 & 0.010 & 2.206 & 1.283 & 0.010 & 1.223 & 1.153 & 0.030 & 1.398 & 1.319 \\
    G-RRT*  & 0.010 & 0.630 & \bfseries0.442 & 0.010 & 1.071 & \bfseries0.462 & 0.010 & \bfseries1.137 & \bfseries0.653 & 0.010 & 1.845 & \bfseries0.691 & 0.010 & 1.165 & \bfseries1.065 & 0.020 & 1.354 & \bfseries1.172 \\
    \bottomrule
    \end{NiceTabular}
    }
    \end{threeparttable}
\end{table*}

%% file: section7/table_manipulation_v2.tex
\begin{table*}[!t]
\centering
\captionsetup[subtable]{justification=centering}
\caption{Initial and final median solution costs for each planner on the manipulation problems described in Section~\ref{sec:manipulation_exp}, along with the median time to find the initial solution ($t_{\text{init}}$, in seconds). Costs are shown after randomized shortcutting and B-spline smoothing. Lower values indicate better performance, with the best cost highlighted in {\setlength{\fboxsep}{1.5pt}\colorbox{green!45}{dark green}} and the second-best in {\setlength{\fboxsep}{1.5pt}\colorbox{green!18}{light green}}.}
\vspace{1mm}
\label{tab:manipulation-problems}
\begin{subtable}[t]{0.73\textwidth}
\centering
\sisetup{table-alignment-mode=format,table-align-text-post=false,table-number-alignment=center,table-format=2.3}
\renewrobustcmd{\bfseries}{\fontseries{b}\selectfont}
\renewrobustcmd{\boldmath}{}
\renewcommand{\arraystretch}{1.3}
\scalebox{0.72}{%
\begin{NiceTabular}[baseline=t]
{
  @{} cl
  S @{\hspace{0.6em}} S @{\hspace{0.6em}} S @{\hspace{0.6em}}
  S @{\hspace{0.6em}} S @{\hspace{0.6em}} S @{\hspace{0.6em}}
  S @{\hspace{0.6em}} S @{\hspace{0.6em}} S @{\hspace{0.6em}}
  S @{\hspace{0.6em}} S @{\hspace{0.6em}} S @{\hspace{0.6em}}
}
\CodeBefore
\rowlistcolors{2}{gray!10, }
\columncolor{white}{1}
\cellcolor{green!45}{10-4,11-5,10-7,10-8,10-10,10-11,11-13,11-14}
\cellcolor{green!18}{11-4,10-5,11-7,11-8,11-10,11-11,10-13,10-14}
\cellcolor{green!45}{20-4,20-5,20-7,20-8,20-10,20-11,14-13,20-14}
\cellcolor{green!18}{12-4,19-5,19-7,19-8,17-10,19-11,20-13,19-14}
\Body
\toprule
&
& \multicolumn{3}{c}{\textbf{bookshelf thin}}
& \multicolumn{3}{c}{\textbf{table pick}}
& \multicolumn{3}{c}{\textbf{box}}
& \multicolumn{3}{c}{\textbf{cage}}\\
\cmidrule(lr){3-5}\cmidrule(lr){6-8}\cmidrule(lr){9-11}\cmidrule(lr){12-14}
&
& {$t_{\text{init}}$} & {$c_{\text{init}}$} & {$c_{\text{final}}$} &
{$t_{\text{init}}$} & {$c_{\text{init}}$} & {$c_{\text{final}}$} &
{$t_{\text{init}}$} & {$c_{\text{init}}$} & {$c_{\text{final}}$} &
{$t_{\text{init}}$} & {$c_{\text{init}}$} & {$c_{\text{final}}$}\\
\midrule
\Block{9-1}{\rotatebox[origin=c]{90}{\textbf{Panda} ($\mathbb{R}^7$)}}
& RRT-C   & 0.010 & 4.860 & 4.860 & 0.010 & 4.921 & 4.921 & 0.010 & 5.427 & 5.427 & 0.020 & 7.845 & 7.845 \\
& RRT*    & $\infty$ & $\infty$ & $\infty$ & $\infty$ & $\infty$ & $\infty$ & $\infty$ & $\infty$ & $\infty$ & $\infty$ & $\infty$ & $\infty$ \\
& RRT*-C  & 0.010 & 4.847 & 4.657 & 0.010 & 4.885 & 4.712 & 0.010 & 5.324 & 4.809 & 0.061 & 7.778 & 7.020 \\
& BIT*    & 0.010 & 4.893 & 4.605 & 0.010 & 4.910 & 4.887 & 0.010 & 5.383 & 4.370 & $\infty$ & $\infty$ & $\infty$ \\
& AIT*    & 0.011 & 4.890 & 4.588 & 0.011 & 4.955 & 4.933 & 0.011 & 5.604 & 4.455 & $\infty$ & $\infty$ & $\infty$ \\
& EIT*    & 0.010 & 4.648 & 4.476 & 0.010 & 4.712 & 4.752 & 0.010 & 4.858 & 4.085 & 7.492 & 12.294 & 10.568 \\
& BLIT*   & 0.010 & 4.979 & 4.638 & 0.010 & 4.908 & 4.657 & 0.010 & 5.590 & 4.344 & $\infty$ & $\infty$ & $\infty$ \\
& AORRTC  & 0.010 & \bfseries4.265 & 4.221 & 0.010 & \bfseries4.326 & \bfseries4.309 & 0.010 & \bfseries3.848 & \bfseries3.793 & 0.010 & 6.819 & 4.725 \\
& G-RRT*  & 0.010 & 4.333 & \bfseries4.219 & 0.010 & 4.365 & 4.316 & 0.010 & 3.927 & 3.800 & 0.010 & \bfseries6.680 & \bfseries4.711 \\
\midrule
\Block{9-1}{\rotatebox[origin=c]{90}{\textbf{Fetch} ($\mathbb{R}^8$)}}
& RRT-C   & 0.401 & 11.522 & 11.522 & 0.010 & 9.318 & 9.318 & 0.010 & 13.039 & 13.039 & 0.231 & 18.372 & 18.372 \\
& RRT*    & $\infty$ & $\infty$ & $\infty$ & $\infty$ & $\infty$ & $\infty$ & $\infty$ & $\infty$ & $\infty$ & $\infty$ & $\infty$ & $\infty$ \\
& RRT*-C  & 6.579 & 12.846 & 12.327 & 0.010 & 8.879 & 6.879 & 0.010 & 12.451 & 8.827 & 2.236 & \bfseries16.682 & 14.562 \\
& BIT*    & $\infty$ & $\infty$ & $\infty$ & 0.012 & 8.326 & 6.845 & 0.289 & 11.091 & 8.879 & $\infty$ & $\infty$ & $\infty$ \\
& AIT*    & $\infty$ & $\infty$ & $\infty$ & 0.030 & 8.461 & 6.862 & 0.409 & 11.465 & 8.671 & $\infty$ & $\infty$ & $\infty$ \\
& EIT*    & $\infty$ & $\infty$ & $\infty$ & 0.010 & 8.070 & 6.563 & 0.042 & 11.007 & 8.148 & $\infty$ & $\infty$ & $\infty$ \\
& BLIT*   & $\infty$ & $\infty$ & $\infty$ & 0.020 & 9.167 & 7.270 & $\infty$ & $\infty$ & $\infty$ & $\infty$ & $\infty$ & $\infty$ \\
& AORRTC  & 0.050 & 11.828 & 6.740 & 0.010 & 7.256 & 5.786 & 0.010 & 11.106 & 7.496 & 0.041 & 18.044 & 10.168 \\
& G-RRT*  & 0.030 & \bfseries10.838 & \bfseries6.377 & 0.010 & \bfseries7.003 & \bfseries5.411 & 0.010 & \bfseries10.726 & \bfseries6.943 & 0.030 & 17.297 & \bfseries7.481 \\
\bottomrule
\end{NiceTabular}
}
\end{subtable}
\hfill
\begin{subtable}[t]{0.25\textwidth}
\centering
\sisetup{table-alignment-mode=format,table-align-text-post=false,table-number-alignment=center,table-format=2.3}
\renewrobustcmd{\bfseries}{\fontseries{b}\selectfont}
\renewrobustcmd{\boldmath}{}
\renewcommand{\arraystretch}{1.3}
\scalebox{0.72}{%
\begin{NiceTabular}[baseline=t]{@{} l S @{\hspace{0.6em}} S @{\hspace{0.6em}} S @{\hspace{0.6em}} @{}}
\CodeBefore
\rowlistcolors{2}{gray!10, }
\cellcolor{green!45}{11-3,11-4}
\cellcolor{green!18}{10-3,10-4}
\Body
\toprule
& \multicolumn{3}{c}{\textbf{dual-arm} ($\mathbb{R}^{14}$)}\\
\cmidrule(lr){2-4}
& {$t_{\text{init}}$} & {$c_{\text{init}}$} & {$c_{\text{final}}$}\\
\midrule
RRT-C   & 0.064 & 13.625 & 13.625 \\
RRT*    & $\infty$ & $\infty$ & $\infty$ \\
RRT*-C  & 0.070 & 13.660 & 11.722 \\
BIT*    & 0.061 & 15.186 & 10.584 \\
AIT*    & 0.733 & 14.134 & 11.203 \\
EIT*    & 0.080 & 13.646 & 10.448 \\
BLIT*   & 0.061 & 13.723 & 10.602 \\
AORRTC  & 0.045 & 12.794 & 9.271 \\
G-RRT*  & 0.030 & \bfseries12.614 & \bfseries7.590 \\
\bottomrule
\end{NiceTabular}
}
\end{subtable}
\end{table*}

%% file: section8/conclusion.tex
\section{Conclusion}
\label{sec:conclusion}

\looseness=-1
In this paper, we build on our earlier introduction of the greedy informed set~\citep{kyaw2022energy}, and provide the first formal analysis of its behaviour within sampling-based planners.
We show that the greedy informed set contains all states along the optimal path whenever the current and optimal solutions share the same homotopy class.
Conversely, when the two paths belong to different homotopy classes, exclusively relying on the greedy informed set can prevent convergence to the global optimum.
These findings highlight the need to appropriately balance greedy exploitation with broader exploration through the full informed set to preserve asymptotic optimality.
To this end, we present G-RRT*, an algorithm that uses bidirectional search to quickly find initial solutions and then leverages the greedy informed set to focus the subsequent search, while preserving asymptotic optimality guarantees.

We evaluate G-RRT* on both abstract planning problems and manipulation tasks across a range of dimensions.
Our results show that G-RRT* finds initial solutions as quickly as RRT-Connect and then uses the remaining planning time to improve them, converging faster than all other asymptotically optimal planners.
It achieves these gains \textit{without} relying on path simplification, unlike AORRTC, and applying simplification as a post-processing step yields even lower solution costs.
Overall, the experiments demonstrate that incorporating greedy informed sampling can substantially accelerate convergence to high-quality solutions in informed planners.

There are several promising directions for future work.
Because G-RRT* currently draws a single sample per iteration, extending greedy informed sampling to batched updates, as in BIT*, could further improve performance.
Additionally, we are investigating ways to define and exploit promising regions of the state space; focusing sampling on these regions may improve both the efficiency and convergence rate of sampling-based motion planners for high-dimensional problems.